\documentclass[letterpaper,10pt]{article} 

\usepackage{opticameet3}

\newcommand\authormark[1]{\textsuperscript{#1}}
\usepackage[round]{natbib}

\usepackage{amsmath,amssymb}
\usepackage[colorlinks=true,bookmarks=false,citecolor=blue,urlcolor=blue]{hyperref} %pdflatex
\usepackage{xcolor}
\usepackage{amssymb}
\usepackage{xifthen}
\usepackage{xparse}
\usepackage{amsthm}
\usepackage{subfig}
\usepackage{tikz}
\usepackage{mathtools} % for \coloneq
\usepackage{pgfplots}
\pgfplotsset{compat=1.18}

\DeclareMathOperator*{\argmin}{arg\,min}
\DeclareMathOperator*{\argmax}{arg\,max}

\NewDocumentCommand{\xy}{O{} O{}}{
    \ifthenelse{\isempty{#1} \AND \isempty{#2}}
    {(x, y)}
    {\ifthenelse{\isempty{#2}}
    {(x^{#1}, y)}
    {\ifthenelse{\equal{#2}{y}}
    {(x^{#1}(#2), y)}
    {\ifthenelse{\equal{#2}{hat}}
    {(\hat{x}^{#1}, y)}
    {(x^{#1} #2, y)}
    }
    }
    }
}

\NewDocumentCommand{\x}{O{} O{}}{
    \ifthenelse{\isempty{#2}}
    {x^{#1}}
    {\ifthenelse{\equal{#2}{hat}}
        {\hat{x}^{#1}}
        {x^{#1}(#2)}
    }
}
\NewDocumentCommand{\omg}{O{} m}{
    \ifthenelse{\isempty{#1}}
    {\Omega #2}
    {\Omega^{#1} #2}
}

\NewDocumentCommand{\Jomg}{O{} m}{
    \ifthenelse{\isempty{#1}}
    {\Omega_1 #2}
    {\Omega_1^{#1} #2}
}
\newcommand{\hy}{\nabla h(y)}

\newcommand{\norm}[1]{\|#1\|}
\newcommand{\lnorm}[1]{\|#1\|}
\newcommand{\opnorm}[1]{\|#1\|_\textup{op}}
\newcommand{\R}{\mathbb{R}}
\newcommand{\tD}{\textbf{D}}
\newcommand{\st}{\text{s.t.}}

\newtheorem{definition}{Definition}
\newtheorem{proposition}{Proposition}
\newtheorem{aproposition}{Proposition}

\newtheorem{remark}{Remark}

\AddToHook{cmd/appendix/before}{%
  \setcounter{proposition}{0}%
  \renewcommand{\theequation}{\thesection.\arabic{proposition}}%
}

\begin{document}

\title{Enhancing Hypergradients Estimation: A Study of Preconditioning and Reparameterization}

\author{
Zhenzhang Ye\authormark{1, 4, $\dagger$}, Gabriel Peyré\authormark{2, $\ddagger$}, Daniel Cremers\authormark{1, 4, $\dagger$}, Pierre Ablin\authormark{3, $\mathsection$}}

\address{\authormark{1} Technical University of Munich\\
\authormark{2} CNRS, ENS - PSL University\\
\authormark{3} Apple\\
\authormark{4} Munich Center for Machine Learning}

\email{$\dagger$ \{zhenzhang.ye, cremers\}@tum.de, $\ddagger$  gabriel.peyre@ens.fr, $\mathsection$ pierre.ablin@apple.com}

\begin{abstract}
Bilevel optimization aims to optimize an outer objective function that depends on the solution to an inner optimization problem.
It is routinely used in Machine Learning, notably for hyperparameter tuning.
The conventional method to compute the so-called hypergradient of the outer problem is to use the Implicit Function Theorem (IFT).
As a function of the error of the inner problem resolution, we study the error of the IFT method. 
We analyze two strategies to reduce this error: preconditioning the IFT formula and reparameterizing the inner problem.
We give a detailed account of the impact of these two modifications on the error, highlighting the role played by higher-order derivatives of the functionals at stake.
Our theoretical findings explain when super efficiency, namely reaching an error on the hypergradient that depends quadratically on the error on the inner problem, is achievable and compare the two approaches when this is impossible.
Numerical evaluations on hyperparameter tuning for regression problems substantiate our theoretical findings.
\end{abstract}

\section{Introduction}\label{sec:intro}

Bilevel optimization, the problem of minimizing an outer function that depends on the solution to an inner problem, has become a standard tool in many areas of machine learning. 
Typical applications include hyperparameter optimization \citep{franceschi2018bilevel, pedregosa2016hyperparameter, bertrand2020implicit}, meta-learning \citep{finn2017model, rajeswaran2019meta} or neural architecture search \citep{liu2018darts}.
It is also used to train implicit deep learning models like deep equilibrium models~\citep{bai2019deep} or networks with optimization layers~\citep{amos2017optnet,blondel2022efficient}.
Large-scale bi-level problems are usually solved using the implicit function theorem (IFT) to compute the gradient of the outer problem relying only on an estimated solution of the inner problem.
In this paper, we challenge this de facto standard by studying variations around this idea, obtained either by preconditioning the IFT formula or by reparameterization, i.e., doing a change of variables.
The fundamental question we tackle is to understand the impact of these new IFT-type formulas on the outer gradient approximation error.
%%%%
\paragraph{Bilevel Optimization.}
We study the bilevel program
\begin{equation} \label{eq:bilevel_problem}
    \min_{y \in \R^{d_y}}~ h(y) = g \xy[\star][y] \quad\st\quad F\xy[\star][y] = 0 
\end{equation}
where $g: \R^{d_x} \times \R^{d_y} \rightarrow \R$ and $F: \R^{d_x} \times \R^{d_y} \rightarrow \R^{d_x}$ are smooth functions.
The energy function $g$ is called \textit{outer function},
and the root-finding problem $F\xy[\star][y]=0$ in the constraint is called \emph{inner problem}.
When $F(x,y)=\nabla_1 f(x,y)$ is the gradient of a convex scalar \textit{inner function} $f\xy$ w.r.t $x$, the inner problem corresponds to the optimization of $f(\cdot,y)$.
In the following, we assume for simplicity that $\x[\star][y]$ is uniquely defined for each $y$.
If this condition does not hold, we assume a consistent selection strategy for a solution $\x[\star][y]$~\citep{arbel2022non}. 

\paragraph{IFT Formula.}

Optimizing over $y$ in Eq.~\eqref{eq:bilevel_problem} typically requires the gradient of the function $h$ w.r.t. $y$, called \textit{hypergradient} $\hy$.
Assuming the Jacobian $\nabla_1 F \xy[\star][y]$ is invertible, the hypergradient can be computed by using the chain rule and Implicit Function Theorem (IFT) \citep{krantz2002implicit}:
\begin{equation}\label{eq:omg_basic}
    \hy  = \Omega\xy[\star][y]
\end{equation}
\begin{equation} \label{eq:hypergrad_basic}
    \text{ where } 
    \omg{\xy}\coloneqq \nabla_2 g \xy + \Psi \xy \nabla_1 g \xy.
\end{equation}
\begin{equation}
\text{ and }
    \Psi\xy \coloneqq -[\nabla_{2} F\xy]^\top[\nabla_1 F\xy]^{-1}.
    \label{eq:Psi_basic}
\end{equation}
Here, $\Psi\xy[\star][y] = \partial \x[\star][y]$ relies on the IFT to compute the Jacobian of the map $y \mapsto \x[\star][y]$.

\paragraph{Approximate Inner Resolution.}

When the exact root $\x[\star][y]$ is available, the IFT formula~\eqref{eq:omg_basic} computes exactly $\nabla h(y)$. 
In practice, one only has access to an approximate root $\hat x$, for instance, $\hat x = \x[k][y]$ can be obtained by running $k$ steps of an iterative resolution method.
The fundamental question studied in this paper is to analyze the error $\Omega(\hat x,y) - \nabla h(y)$ as a function of the inner problem error $\hat x - \x[\star][y]$. 
Several strategies like warm starting \citep{bai2022deep,thornton2023rethinking} and amortized learning \citep{amos2023tutorial} have been proposed to reduce this error $\hat x - \x[\star][y]$.
Finally, \citet{ramzi2021shine} uses the Hessian approximation learned by a Quasi-Newton method during the inner problem resolution to have a better estimation of the Jacobian $\Psi$.
Still, as long as this error is non-zero, directly using $\Omega(\hat x,y)$ as a proxy for $\nabla h(y)$ leads to an inaccurate estimation of the hypergradient,
which could cause an accumulated error when optimizing the function $h$ \citep{devolder2014first}. 
Even with some simple convex functions, the error $\hat x - \x[\star][y]$ can be amplified on the hypergradient estimation \citep{mehmood2021differentiating}.
We question the direct use of $\Omega$ and propose alternate formulas $\tilde\Omega$ based on preconditioning or reparameterization which might lower the error $\tilde\Omega(\hat x,y) - \nabla h(y)$.

\paragraph{Preconditioning and Reparameterization.}

Many methods accelerate the convergence of $\x[k][y]$ toward $\x[\star][y]$ by preconditioning each step with a linear mapping \citep{golub2013matrix, spielman2004nearly}.
When $F=\nabla_1 f$, the intuition is that this preconditioning should capture the curvature of $f$, hence it should be close to the inverse of the Hessian of $f$, which corresponds to Newton's method.
Finding an efficient preconditioner is a trade-off between the approximation of the Hessian and the ease of inversion.
Another widely used strategy is reparemeterization, i.e., to perform a change of variable $z=\phi(x,y)$ over the inner problem, and perform the inner optimization over the $z$ variable \citep{salimans2016weight, kingma2013auto, moins2023reparameterization}.
From an optimization perspective, reparameterization is closely related to preconditioning, where the preconditioner depends on the Jacobian of $\phi(\cdot,y)$. 

\paragraph*{Contributions and Paper Organization.}

In this paper, we propose a unified study of the IFT-type formula to estimate $\nabla h$.
We study in particular formulas derived by preconditioning and reparameterization:
\begin{itemize}
    \item In Sec.~\ref{sec:eff}, we characterize the error of the hypergradient estimation when using $\xy[k][y]$. 
    The Jacobian of $\Omega$ (Eq.~\eqref{eq:omg_basic}) w.r.t. $x$ determines the error decay of estimation.
    We introduce the concept of super efficiency where the Jacobian is $0$, leading to a hypergradient estimation that decays quadratically with the error $\hat{x} - x^\star$.
    \item In Sec.~\ref{sec:proposed}, we analyze the impact of two strategies, preconditioning and reparameterization, on the hypergradient estimation.
    We describe cases where each strategy achieves super efficiency.
    \item In Sec.~\ref{sec:comparison_between_methods}, we compare these two strategies in different settings.
    Our results hint at the superiority of preconditioning, while reparameterization could be a better choice in certain corner cases. 
    \item Sec.~\ref{sec:numerics} presents numerical experiments illustrating this paper's theory.
\end{itemize}

\paragraph{Related Work.}
Problem~\eqref{eq:bilevel_problem} is usually solved with an iterative algorithm,
like gradient-based algorithms \citep{beck2017first}, Newton's method \citep{boyd2004convex}, and second order Quasi-Newton methods \citep{shanno1970conditioning}.
Two main approaches can be used to compute the gradient of $h$: automatic and implicit differentiation.
Assume that we have access to an iterative strategy that builds the sequence $x^i(y)$ for $i=0\dots k-1$ that converges to $x^\star(y)$, like the power method, and that we use the last iterate $x^k$ as an approximation to $x^\star$.
Automatic differentiation \citep{griewank2008evaluating} computes an approximation of $\nabla h(y)$ as $\partial_y h(x^{k}(y))$, where the differentiation is done through the iterates of the algorithm. 
It does so by leveraging the chain rule repeatedly to the elementary operations and functions in the reverse mode \citep{christianson1994reverse}.
\cite{gilbert1992automatic} analyzes its behavior in the context of the iterative procedure.
It has become popular in several bilevel applications \citep{domke2012generic, franceschi2017forward,mehmood2020automatic, bolte2022automatic}.
This approach requires storing each iterate $x^i(y)$ in memory, which makes it impractical for iterative procedures with thousands of iterations.
Implicit differentiation \citep{bengio2000gradient} overcomes this drawback, using only the last iterate $x^k$.
It is the approach of choice for problems such as deep equilibrium network \citep{bai2019deep}, non-smooth problems \citep{bolte2021nonsmooth}, and hyperparameter tuning \citep{franceschi2018bilevel}.

The preconditioning strategy is common for optimizing a function \citep{becker2012quasi, pock2011diagonal}.
Some typical preconditioners include diagonal preconditioner, incomplete Cholesky factorization \citep{golub2013matrix}, and Laplacian preconditioning \citep{spielman2004nearly}.
Although various works analyze the convergence of the technique \citep{benzi2002preconditioning}, the impact of preconditioning on the hypergradient estimation is unclear.

Many methods can be viewed as reparameterization. \cite{salimans2016weight} introduce a reparameterization of the weight vectors in a neural network to accelerate the training process.
\cite{kingma2013auto} apply the reparameterization trick on variational autoencoders to allow the backpropagation on a random node.
When optimizing a constrained problem, reparameterization is a common strategy to deal with simple constraints \citep{jorge2006numerical}.
In this work, we study whether reparameterization can improve hypergradient estimation.
\begin{table}[h!]
\centering
\begin{tabular}{|c|c|}
    \hline
     Variable & Definition  \\ \hline
     $\norm{\cdot}, \norm{\cdot}_\text{op}$ & Euclidean, operator norm\\
     $I_d$ & $d \times d$ identity matrix \\
     $F_1(\cdot, \cdot)$ & Jacobian of $F$ w.r.t the first variable\\
     $F_{11}(\cdot, \cdot)$ & Jacobian of $F_1$ w.r.t. the first variable\\
     \hline
\end{tabular}
\caption{Notations}
\label{tab:notations}
\end{table}

\paragraph{Notation.}
To simplify notations, we use the subscript to denote differentiation w.r.t. that variable, i.e. we use for short $F_1\xy = \nabla_1 F\xy$, $F_2\xy = \nabla_2 F\xy$.
The second-order derivative of $g\xy$ is denoted by $g_{12}\xy = \nabla^2_{12} g \xy \in \R^{d_x \times d_y}$, similarly for $g_{11}\xy \in \R^{d_x \times d_x}$.
A detailed table of notations is shown in Table~\ref{tab:notations}.

\section{Error Analysis and Super Efficiency} \label{sec:eff}

n this section, we study the structure of the hypergradient estimation problem.
We consider a generic formula $\tilde \Omega(x,y)$, where $\tilde{\Omega}: \R^{d_x} \times \R^{d_y} \rightarrow \R^{d_y}$ is a function that approximates the hypergradient $\nabla h(y)$.
The prototypical example is the IFT formula~\eqref{eq:hypergrad_basic}. 
The most basic requirement on the formula $\tilde{\Omega}$ is that it is \emph{consistent}, which means that it correctly recovers the hypergradient if we set $x=\x[\star][y]$.
\begin{definition}[Consistency]
A formula $\tilde\Omega$ is said to be consistent if it satisfies 
$$
    \forall y, \quad
    \tilde \Omega(\x[\star][y],y) = \nabla h(y).
$$
\end{definition}
By definition, the IFT formula $\Omega$ is consistent.
Assuming that $\tilde\Omega$ is a smooth $C^1$ map, one can control the impact of an approximate computation $\hat x$ of $\x[\star][y]$ by doing a Taylor expansion, as stated in the following proposition. 

\begin{proposition}[Hypergradient approximation]\label{prop:efficiency}
   If $\tilde\Omega$ is $C^1$ and consistent, then for all $\hat{x}$ and $y$
    \begin{align*}
    &\norm{
        \tilde{\Omega}(\hat x,y) - 
        \nabla h(y) 
    }
    \leq 
    C_y \norm{ \x[\star][y]-\hat x }
    + 
    \mathcal{O}( \norm{ \x[\star][y]-\hat x }^2 ), \\
     &   \quad \text{where} \quad  
        C_y \coloneqq C_y(\tilde\Omega) \coloneqq \opnorm{\tilde{\Omega}_1(\x[\star][y], y)}.
    \end{align*}
\end{proposition} 

This simple result exposes the fact that, at first order, controlling the estimation error on the hyper-gradient requires the control of $C_y(\tilde\Omega)$, which is the norm of the Jacobian w.r.t. $x$ of the formula $\tilde\Omega$.
It is the fundamental quantity that needs to be analyzed to understand the efficiency of a formula.
We see that the hypegradient estimation error diminishes with $C_y(\tilde\Omega)$: a good hypergradient estimator should therefore strive to make this constant as small as possible.
Of particular interest is when this term is 0:

\begin{definition}[Super efficiency~\citep{ablin2020super}]
    If $C_y(\tilde\Omega)=0$, the formula $\tilde\Omega$ is said to be \emph{super efficient}.
    Equivalently, according to Prop.~\ref{prop:efficiency}, the estimation error on the hypergradient has a quadratic decay with respect to the inner problem resolution error.
\end{definition}

The following proposition computes $\tilde\Omega_1$ in the case of the IFT formula $\Omega$ (Eq.~\eqref{eq:omg_basic}).

\begin{proposition}[Jacobian of estimation]
Assuming $g$ and $F$ are smooth, one has 
\begin{equation*}
    \Jomg{\xy} = g_{21}{\xy} + \Psi_1 \xy g_1 \xy + \Psi \xy g_{11} \xy, 
\end{equation*}
\begin{equation}
\begin{aligned}
&\text{ with } \Psi_1 \xy = -[F_{12} \xy ]^\top[F_1 \xy[]{}]^{-1}\\
&+ [[F_1 \xy[]{}]^{-1} F_{11} \xy[]{} [F_1 \xy[]{}]^{-1} F_2 \xy[]{}]^\top .
\end{aligned}
\label{eq:JPsi} 
\end{equation}
\end{proposition}

\paragraph{Efficiency on the Inner Problem.}

While we phrase all our results directly in terms of the estimation error of the hypergradient $\nabla h$, the core of our analysis aims at controlling the error on the Jacobian $\partial x^\star(y)$ of the inner variable. 
The estimation of this Jacobian of interest in itself beyond just bilevel programming.
The following proposition shows the relation between the error of the estimator $\Omega$ of the hypergradient and the estimator $\Psi$ of the inner problem.

\begin{proposition}[IFT efficiency]\label{prop:outer-inner}
    One has
    $$
        C_y(\Omega) \leq \| g_{21} + \partial x^\star(y) g_{11} \|_{\infty}  + 
             \norm{g_1}_\infty C_y(\Psi)
    $$
    where $\norm{H}_\infty := \sup_{x,y} \norm{H(x,y)}_\text{\upshape op}$.
    Hence, if $g$ is of the form $g(x, y) = ax + m(y)$, then $\Omega$ is super-efficient if $\Psi$ is super-efficient.
    If, furthermore, $F $ is of the form $A x + M(y)$, then $\Omega$ is super-efficient.
\end{proposition}

Prop.~\ref{prop:outer-inner} describes how the non-linearity of the outer problem impacts the efficiency of $\Omega$.
In general, $C_y(\Omega)$ is not $0$; the following section aims to design alternate formulas $\tilde\Omega$ so that $C_y(\tilde{\Omega})$ as small as possible.

\section{PROPOSED STRATEGIES}\label{sec:proposed}
%\section{Proposed Strategies}\label{sec:proposed}

We detail here two classes of formula $\tilde\Omega$ which are consistent by design, and which hopefully improve the efficiency constant $C_y(\tilde\Omega) = \opnorm{ \tilde\Omega_1\xy[\star][y] }$ over the vanilla IFT formula $\Omega$. 
These two strategies operate either by directly preconditioning the IFT formula or by applying the IFT to a reparameterized problem. 

\subsection{Preconditioning} \label{sec:precond}

We consider an invertible matrix $P\xy \in \R^{d_x \times d_x}$ and perform an update on a given $x$:
\begin{equation}
    \tilde x \coloneqq \x - P\xy^{-1}F\xy.
    \label{eq:x_hat}
\end{equation}
With a proper choice of $P\xy$, $ \tilde x $ becomes closer to $\x[\star](y)$ than $x$.
For instance, if $F$ is the gradient of a convex function, a simple choice for $P$ is a large enough symmetric positive matrix.
Therefore, rather than using $x$, we estimate the hypergradient with $ \tilde x$:
\begin{equation}
    \omg[P]{\xy} \coloneqq \omg{(\tilde{x}, y)} = \omg{\xy[][-P\xy^{-1}F\xy]}.
\label{eq:omgP}
\end{equation}

\paragraph{Preconditionning Efficiency.}

We now turn to the analysis of the error of this new estimator, by relating it to the Jacobian $\Omega_1$ of the IFT formula $\Omega$.

\begin{proposition}[Preconditioned estimation]
    $\Omega^P$ is consistent and 
    \begin{equation}\label{eq:JomgP}
        \Jomg[P]{\xy} = \Omega_1(\tilde{x}, y)
        E^P\xy
        % (I_{d_x} - P(y)^{-1}F_1\xy[\star][y]).
    \end{equation}
    where $E^P \xy \coloneqq I_{d_x} - [P\xy]^{-1} F_1 \xy + [P\xy]^{-1} P_1\xy[][] [P\xy]^{-1} F\xy $.
\label{prop:precond_est}
\end{proposition}
\begin{proof}
    We give a sketch proof and the full one is available in the appendix.
    The root of $F$ means that $F\xy[\star][y]=0$.
    Plugging it into Eq.~\eqref{eq:x_hat} leads to $x^\star - [P\xy[\star][y]]^{-1}F\xy[\star][y] = x^\star$, for any $P$.
    With Eq.~\eqref{eq:omgP}, we have $\omg[P]{\xy[\star][y]} = \omg{\xy[\star][y]}$.
    Deriving $\Omega^{P}$ w.r.t $x$ using the chain rule and using $F\xy[\star][y] = 0$, we have $\Omega_1^P\xy[\star][y]$.
\end{proof}

This proposition shows that the improvement brought by preconditioning with respect to the IFT formula is precisely captured by $E^P$.
Importantly, this term nicely simplifies at $\x[\star][y]$, giving $E^P\xy[\star][y]=I_{d_x} - P\xy[\star][y]^{-1} F_1 \xy[\star][y]$.
As long as $\opnorm{E^P\xy[\star][y]} < 1$, preconditioning improves the hypergradient estimation compared to the vanilla IFT.
In the trivial case where $P(x, y) \equiv 0$, $E^P \xy[\star][y] = I_{d_x}$, and $\Jomg[P]{\xy[\star][y]} = \Jomg{\xy[\star][y]}$.
More generally, reducing  $\Jomg[P]$ amounts to finding a good approximation to the (inverse) Hessian, which is studied next.

\paragraph{Newton Preconditioners.}

To achieve super-efficiency, the goal is to design $P$ so that $C_y(\Omega^P)=0$.
This can be achieved by using a Newton-type strategy.

\begin{proposition}[Newton-like preconditioner]\label{prop:ideal_precond}
    For $P(x,y) = F_1(x,y)$, 
    $\Omega^P$ is super-efficient.
\end{proposition}

\begin{proof}
In this case, we have $E^p\xy[\star][y] = 0$, which implies $\Omega^P_1 (x^\star(y), y)=0$. 
\end{proof}

Note that Eq.~\eqref{eq:x_hat} with the ``ideal'' preconditioner $P\xy$ corresponds to the iterates of Newton's method.
A chief advantage of such a strategy is that it is super-efficient independent of the choice of the inner function $f$ and the outer function $g$.
This however comes at a price, because applying this formula requires computing a Hessian and solving a linear system.
These operations are computationally expensive and even intractable in large-scale problems.
In practice, one usually leverages a cheap approximation of the inverse of $F_1$.
When $F_1$ is diagonally dominant, an efficient solution, the Jacobi preconditioner, only retains the diagonal of $F_1$.

\subsection{Reparameterization}

Another classical way to accelerate optimization algorithms is by \emph{reparameterization}, with a well-chosen change of variable. 
We denote $x=\phi(z, y)$ a \emph{surjective} change of variable, so that the initial problem~\eqref{eq:bilevel_problem} is equivalent (in the case where $F=\nabla_x f$ for concreteness) to
\begin{equation}\label{eq:change-var}
    \min_y h(y) = \tilde g(z(y),y)
    \:\st\: 
    z(y) \coloneqq \argmin_z \tilde f(z,y), 
\end{equation}
where $\tilde g(z(y),y) \coloneqq g(\phi(z(y),y),y)$
and $\tilde f(z,y) \coloneqq f(\phi(z,y),y)$. 
Note that we reformulate the bilevel optimization in Eq.~\eqref{eq:bilevel_problem} with the constraint of a minimization problem for the sake of concreteness.
For the general case of a constraint $F(x,y)=0$, $F(x,y)$ should be changed into $\tilde F(z,y) \coloneqq \phi_1(z, y)^\top F(\phi(z, y), y) = 0$. 
We emphasize that the change of variable should be applied to the inner function $f$.
Although the following discussion only requires $\phi$ to be surjective, we assume $\phi$ to be bijective for the sake of simplicity. Its inverse on $x$ is denoted as $\phi^{-1}$.
A crucial point is that, even though the bilevel program is invariant under this change of variable, the IFT formula $\Omega(x,y)$ is not. In the following, we denote $\Omega^\phi$ the formula obtained by replacing $(g, F)$ by $(\tilde g,\tilde F)$.

\begin{proposition} $\Omega^{\phi}$ is consistent and 
\begin{align}   
\omg[\phi]{\xy} & \coloneqq g_2\xy + \Psi^{\phi}\xy g_1\xy,  \label{eq:omg_phi} \\
    \label{eq:Psi_phi}
    \Psi^{\phi}\xy & \coloneqq \phi_2(z,y)^\top - U^\phi\xy^\top [V^\phi\xy]^{-1}, \\ \label{eq:U_def}
    &U^\phi(x, y) \coloneqq F_2 +F_1 \phi_2+ \phi_1^{-1} (\phi_{21} F), \\
    &V^\phi(x, y) \coloneqq \phi_1^{-\top}(\phi_{11} F) \phi_1^{-1} + F_1.
    \label{eq:V_def}
\end{align}
with $z = \phi^{-1}(x, y)$. Additionally, denoting $x^\star \coloneqq x^\star(y)$:
\begin{align}
   \Jomg[\phi](x^\star,y) &= D(y) + \Psi^{\phi}_1(x^\star,y) g_1(x^\star,y),
    \label{eq:Jomg_phi} \\
    D(y) & := g_{21} (x^\star, y) + [\partial x^\star]^\top g_{11} (x^\star, y), 
    \label{eq:Dxy}\\
    \Psi^{\phi}_1(x^\star,y) &= \Psi_1(x^\star, y) + C^\phi(y), \\
    C^\phi(y) &= W^\phi (y)+ S^\phi (y) + T^\phi(y), 
\end{align}
\begin{equation}
\begin{aligned}
W^\phi(y)  & \coloneqq - F_1^{-1} \phi_1^{-\top} \phi_{12}^\top F_1,\\
S^\phi(y) & \coloneqq [\phi_1^{-\top}(\phi_{11}F_1)\phi_1^{-1}F_1^{-1}\phi_2]^\top,\\
T^\phi(y) & \coloneqq 
[F_1^{-1}(\phi_1^{-\top}(\phi_{11}F_1)\phi_1^{-1})F_1^{-1} F_2]^\top.
\end{aligned}
\label{eq:JPsi_phi}
\end{equation}
\label{prop:rep_est}
\end{proposition}

\begin{proof}
    We omitted $\xy$, $(z, y)$ in the right side of  Eq.~\eqref{eq:U_def}-~\eqref{eq:V_def} and $(x^\star, y)$, $(z^\star, y)$ in the right side of Eq.~\eqref{eq:JPsi_phi} for simplicity.
    A more detailed derivation can be found in the Supplementary.
    $g(\phi(z(y), y), y)$ contains three $y$s.
    The derivative w.r.t. the latter two can be computed by using the chain rule directly.
    The derivative w.r.t. the first $y$ requires the IFT.
    The gradient of $\tilde f(z, y)$ w.r.t. $z$ corresponds to a new root-finding problem.
    Viewing it as a function on $(z^\star(y), y)$, we can apply IFT to get the derivative of the map $y \mapsto z^\star(y)$.
    Using this derivative and the chain rule, we have $\Psi^{\phi}\xy$.
    %%%%%
    Optimality of $x^\star(y)$ reads $F(x^\star, y)=0$. Plugging it into Eq.~\eqref{eq:Psi_phi}, we have $U = F_2 + F_1 \phi_2$ and $V = F_1$. Computing $\Psi^{\phi}(x^\star, y)$ by these two equations yields $\Psi^{\phi}(x^\star, y) = \Psi(x^\star, y)$. Therefore, we have $\Omega^{\phi}$ is consistent.
    The Jacobian of $\Omega^{\phi}(x, y)$ w.r.t. $x$ follows the chain rule directly. 
\end{proof}

\begin{remark}[Identity map]
    Despite the complexity of these expressions, a sanity check is to verify that if $\phi(z,y)=z$, then the above formulae simplify to the IFT ones, since $\Omega^\phi=\Omega$.
    Indeed, in that case, $\phi_2(z, y) = \phi_{21} (z, y) = \phi_{11} (z,y) = 0$, so that in Eq.~\eqref{eq:Psi_phi}, one has $\Psi^\phi \xy = \Psi \xy$ for any $\xy$.
\end{remark}

%%%
\paragraph{Super-Efficient Reparameterization in 1-D.}

Using Prop.~\ref{prop:rep_est}, one can seek for a change of variable $\phi$ in order to minimize $C_y(\Omega^\phi)$. In particular, it turns out that achieving super-efficiency is equivalent to solving a high-dimensional second-order partial-differential equation in $\phi$, which, to the best of our knowledge, has no explicit solution.
The following proposition shows this second-order differential equation in a simple scalar case.

\begin{proposition}
    Assume $x,y \in \mathbb{R}$ and 
    that $g(x, y)$ and $F(x, y)$ are linear w.r.t. $x$ but arbitrary on $y$. Then $\Omega^\phi$ is super-efficient if and only if for all $y$, 
    \begin{equation}
        \frac{\phi_{12}}{\phi_1} - \frac{\phi_2 \phi_{11}}{[\phi_1]^2}
        - \frac{F_2 \phi_{11}}{F_1 [\phi_1]^2}
        = \frac{g_{12}}{g_1} -\frac{F_{12}}{F_1}  
    \label{eq:exp_1d}
    \end{equation}
    where $\phi_{12} = \phi_{12}(z^\star(y),y)$ (and similarly for other terms).
    % ommited the dependency in $(x,y)$
\end{proposition}

In the (admittedly singular) case where the outer function is affine, the above formulation can be leveraged to construct super-efficient reparameterizations $\phi(z)$ depending only on the variable $z$. 

\begin{proposition}\label{prop:supereffi-1d}
If $g$ is affine on $x$ (see Prop.~\ref{prop:outer-inner}), then super-efficient reparameterizations $\phi(z, y)=\phi_0(z)$ exist and define locally a 2-parameters family of maps.
If furthermore $F$ is linear of $x$, i.e. $F\xy = a(y) x + b$, where $a: \R \rightarrow \R$ and $b \in \R$, these super-efficient maps are of the form $\phi_0(z) = \alpha e^{\beta z}$ for $(\alpha,\beta) \in \R^2$.
\end{proposition}

\begin{proof}
    Eq.~\eqref{eq:exp_1d} boils down to a second-order ODE of a scalar variable on $\phi_0(z)$. 
    Under the smoothness hypothesis on $(g,F)$, one can apply the Cauchy-Lipschitz theorem to ensure the local existence of a super-efficient change of variable.
    In the simple case where $F$ is affine, the equation is simple enough to be solved, giving the advertised formula for $\phi$.
\end{proof}

If one drops the constraint that $\phi(z,y)$ only depends on $z$, then another super-efficient reparameterization is $\phi(z, y) = z / a(y)$, which coincides with the preconditioning step.

\subsection{Separable Localized Reparameterizations} \label{sec:sep_rep}

\paragraph{Localized Reparametrizations.}

A difficulty with the computation of a reparameterized formula $\Omega^\phi$ is the necessity to be able to compute the inverse map $\phi^{-1}(\cdot,y)$. To allow for easily inversible maps, we introduce ``localized'' changes of variable of the form $\phi(z,y)=\psi_{x,\bar y}(z,y)$ which depend on extra fixed parameters $(x,\bar y)$. The resulting localized formula is then
$$
    \Omega_{\text{loc}}^\psi(x,y) \coloneqq
    \Omega^{\psi_{x,y}}(x,y).
$$
Note that while it leverages at each $(x,y)$ a change of variable formula, it is not globally an estimator of the form $\Omega^\phi$ for some fixed $\phi$. The following section highlights its versatility in the context of the computation of separable change of variables. Despite being more general, the following proposition shows that the computation of its efficiency constant $C_y$ still boils down to the one of a classical change of variable.

\begin{proposition}
    The estimator $\Omega_{\text{\upshape loc}}^\psi$ is consistent and one has
    $$
        C_y( \Omega_{\text{\upshape loc}}^\psi )
        =
        C_y( \Omega^{\psi_{x^\star(y),y}} ).
    $$
\end{proposition}

\paragraph{Separable Localized Reparameterization.} Even with a localized change of variable, designing efficient change of variable remains difficult, even in the 1D case as shown in Eq.~\eqref{eq:exp_1d}.
To ease this task, inspired by the separation of variables used while solving PDE, we relax this problem by assuming a separable form for $\psi_{x, \bar{y}}(z, y)$:
\begin{equation}
    \psi_{x, \bar{y}}(z, y) = R(x, y) Q(z, \bar{y}) + x,
\label{assp:def_sep_rep}
\end{equation}
where $R(x, y) \in \R^{d_x \times d_x}$ and $Q(z, \bar{y}) \in \R^{d_x}$.
To ensure that $\psi_{x, \bar{y}}(z, y)$ is bijective, we impose that $R(x, y)$ is invertible and $Q(\cdot,\bar y)$ is bijective.
The following proposition explicitly gives the quantities involved in Prop.~\ref{prop:rep_est}.

\begin{proposition}\label{eq:JPsi_phi_sep}
    Let $\phi = \psi_{x, \bar{y}}$ as in Eq.~\eqref{assp:def_sep_rep},    
    one has:
\begin{equation*}
\begin{aligned}
& W^{\phi}(y) = - F_1^{-1} R^{-\top}Q_1^{-\top} Q_1 R_2^{\top}F_1,\\
& S^{\phi}(y) = [R^{-\top}Q_1^{-\top}( Q_{11} R F_1)Q_1^{-1} R^{-1}F_1^{-1}(QR_2)]^\top, \\
& T^{\phi}(y) =[F_1^{-1}(R^{-\top}Q_1^{-\top}(Q_{11} RF_1)Q_1^{-1} R^{-1})F_1^{-1} F_2]^\top,
\end{aligned}
\end{equation*}
for any $(x, \bar{y})$, where $R = R(x, y)$, $Q = Q(z^\star(y), \bar{y})$ with $\psi_{x, \bar{y}}(z^\star(y), y) = x^\star(y)$ (and similarly for other terms).
\end{proposition}

This shows that the search for super-efficient separable parameterization leads to a simpler differential equation on $(R,Q)$ than the original equation on $\phi$.
However, it is still non-linear and there exists no analytical solution in general.
We now show that this efficiency is however always improving as one approaches a Newton-type class of changes of variables and the outer problem is close to being affine.

\paragraph{Efficiency of $\Omega^{\psi}_{\normalfont{\text{loc}}}$.}
Mimicking the preconditioner studied in Prop.~\ref{prop:ideal_precond}, we now show that a local change of variable based on a Newton step defines a super-efficient reparameterization. Note however that, in contrast to Prop.~\ref{prop:ideal_precond}, in this case, super-efficiency is only possible in the case of an affine outer problem (so it only improves the efficiency of the inner problem estimation).

\begin{proposition}[Newton-like reparameterization]
    We assume $g$ is of the form $g(x, y) = ax + m(y)$(see Prop.~\ref{prop:outer-inner} for a discussion).
    Let $\psi_{x, \bar{y}}$ as in Eq.~\eqref{assp:def_sep_rep} and let $F(x, y)$ be bijective on $x$ for all $y$. 
    For $R(x, y) = [F_1(x, y)]^{-1}$, $Q(z, \bar{y}) = -F(z, \bar{y})$, $\Omega^{\psi}_{\text{\upshape loc}}$ is super efficiency.
\label{prop:ideal_sep_rep}
\end{proposition}
\begin{proof}
    The super efficiency requires to examine $C_y(\Omega^\psi_{\text{loc}})$.
    Because $(x, \bar{y})$ are extra parameters not variables, $\Omega^\psi_{\text{loc}}$ has the same formula as in Prop.~\ref{prop:rep_est}.
    While computing its Jacobian, using the fact of $F(x^\star(y), y) =0$, we can still get the same results as in Prop.~\ref{eq:JPsi_phi_sep}.
    Plugging the special choice of $R$ and $Q$, we get $\Psi_1^{\psi_{x^\star(y), y}}(x^\star(y), y)=0$.
    $C_y(\Omega^\psi_{\text{loc}}) $ is thus $0$ because $g$ is affine.
\end{proof}

\begin{remark}
The above proposition reveals that the reparameterization is never equivalent to the preconditioning strategy.
In the case of a quadratic inner problem, the Newton-like preconditioner in Prop.~\ref{prop:ideal_precond} achieves super-efficiency with any $g$, while the Newton-like reparameterization requires an assumption on $g$.
\end{remark}

This proposition should be understood as providing heuristic guidance to design $R$ and $Q$, which we exploit in the numerical simulations in Sec.~\ref{sec:numerics}.
The following more general proposition states that for generic localized separable changes of variable, the efficiency constant is upper-bounded by error terms measuring how far the change of variable is different from the above-mentioned Newton-type reparameterization.

\begin{proposition}
    Let $x^\star\coloneqq x^\star(y)$ and $\psi_{x^\star, y}$ be defined as in Eq.~\eqref{assp:def_sep_rep} and $z^\star$ satisfy $\psi_{x^\star, y}(z^\star, y) = x^\star$.
    Denoting $E^Q(y) \coloneqq Q(z^\star, y)  + F(x^\star, y)$,
    $E^{Q_1}(y) \coloneqq Q_1(z^\star, y) + F_1(x^\star, y)$,
    $E^{Q_{11}}(y) \coloneqq Q_{11}(z^\star, y) + F_{11}(x^\star, y)$,
    $E^R(y) \coloneqq R(x^\star, y) - [F_1(x^\star, y)]^{-1}$, 
    $E^{R_2}(y) \coloneqq R_2(x^\star, y) + [F_1(x^\star, y)]^{-1}F_{21}(x^\star, y)[F_1(x^\star, y)]^{-1}$.
    We have that 
    $
        C_y(\Psi^{\psi}_{\text{\upshape loc}}) = \mathcal{O}(\norm{E^Q}_{\text{\upshape op}},\norm{E^{Q_1}}_{\text{\upshape op}}, \norm{E^{Q_{11}}}_{\text{\upshape op}}, \norm{E^R}_{\text{\upshape op}}, \norm{E^{R_2}}_{\text{\upshape op}}).
    $
\end{proposition}

Note that this proposition shows $(R,Q)$ should be close to Newton-type functionals, but their higher-order derivatives should also be close, which highlights the difficulty of designing efficient changes of variables. 

\section{Comparison Between Methods}
\label{sec:comparison_between_methods}

As it should be clear from the above analysis, super-efficiency is out of reach for cases of practical interest. %
We thus focus on comparing the efficiency constant $C_y$ of the different strategies. 
We first focus our analysis on the preconditioning $\Omega^P$ with $P(x,y)$  and the reparameterization $\Omega^\phi$ with an arbitrary smooth and bijective functional $\phi(z, y)$.
The following proposition, which is leveraged in special cases below, gives a general formula to compare the efficiency constants.

\begin{proposition}[Comparison of the two methods]
    Let $\phi$ be smooth and bijective, one has
    \begin{align}
        [C_y(\Omega^\phi)]^2 - [C_y(\Omega^P)]^2 &\geq \langle U_+ v_P, U_- v_P \rangle \\
        [C_y(\Omega^P)]^2 - [C_y(\Omega^\phi)]^2 &\geq \langle V_+ v_\phi, V_- v_\phi \rangle
        \label{eq:comp} 
    \end{align}
    where, for $E^P$ as in Eq.~\eqref{eq:JomgP}, $\Psi_1^\phi$ and $D$ as in Eq.~\eqref{eq:Dxy}, 
    \begin{align*}
        U_\pm &\coloneqq D \pm D E^P + \Psi_1^{\phi}g_1 \pm \Psi_1 g_1 E^P, \\
        V_\pm &\coloneqq D E^P \pm D + \Psi_1 g_1 E^P \pm \Psi_1^{\phi}g_1, \\ 
        v_{\omega} &\coloneqq \argmax_{\lnorm{u}=1} \lnorm{\Jomg[\omega]{\xy[\star][y]}u}
        \text{ for } \omega \in \{P,\phi\}, 
    \end{align*}
The dependencies on $\xy[\star][y]$ on the right side are omitted.
\label{prop:comp_two}
\end{proposition}

%%%
\paragraph{Asymptotic Analysis of Preconditioning Superiority.}

As a consequence, the following proposition shows that as the preconditioning quality $\delta$ is small enough, then the preconditioning formula $\Omega^P$ is necessarily better than the reparameterization one $\Omega^\phi$.
This should not be surprising since when $\delta=0$, the preconditioning is super-efficient, while the reparameterization is not necessarily super-efficient.
\begin{proposition}
For $\delta \coloneqq \norm{P\xy[\star][y] - F_1\xy[\star][y]^{-1}}_\infty$, we have
$$
    [C_y(\Omega^\phi)]^2 - [C_y(\Omega^P)]^2 \geq \norm{(D+ \Psi^{\phi}_1g_1) v_P}^2 + o(\delta),
$$
all the terms being evaluated at $\xy[\star][y]$ and $v_P$ as defined in Prop.~\ref{prop:comp_two}.
\label{prop:delta_eq}
\end{proposition}
Note that the term $(D+ \Psi^{\phi}_1g_1) v_P$ does not cancel in general, which shows that the preconditioning strategy should be favored when a suitable preconditioner (that makes $\delta$ small) is available.
A suitable preconditioner can also be leveraged in localized reparameterization by $\psi_x(z, y) = [P(x, y)]^{-1}z$.
However, the improvement from the reparameterization is less than the preconditioning because of $(D+ \Psi^{\phi}_1g_1) v_P$.

%%%
\paragraph{Asymptotic Analysis of Separable Localized Reparameterization Superiority.}

On the other hand, if $P$ does not approximate very well $F_1^{-1}$, it might be possible that the reparameterization $\Omega^\phi$ is better than $\Omega^P$ if $\phi$ is well designed. 
Extreme cases were already given for 1-D problems in Prop.~\ref{prop:supereffi-1d}, where super-efficient changes of variables are detailed.
Additionally, in the case of the separable localized reparameterization, the following proposition shows that if $g$ is close to being affine on $x$ (as already analyzed in Prop.~\ref{prop:outer-inner}), then the reparameterization $\Omega^\psi_{\text{loc}}$ could be a better choice than the preconditioning $\Omega^P$.

\begin{proposition}
For $\sigma \coloneqq \norm{g_1(x^\star(y), y)}_\infty C_y(\Psi_{\text{\upshape loc}}^\psi),$
% $\sigma \coloneqq \norm{D+E^\psi g_1}$,
\begin{equation*}
\begin{aligned}
    &[C_y(\Omega^P)]^2 - [C_y(\Psi^\psi_{\text{\upshape loc}})]^2\\
     \geq &\norm{(D + \Psi_1 g_1) E^Pv_\phi}^2 - \norm{Dv_\phi}^2 +o(\sigma),
\end{aligned}
\end{equation*}
where $v_\phi$ is defined in Prop.~\ref{prop:comp_two} with $\phi = \Psi_{\text{\upshape loc}}^\psi$.
\end{proposition}

In conclusion, the preconditioning strategy outperforms in general as long as a valuable preconditioner is available.
Otherwise, a well-designed reparameterization could be a better choice when approximating $F_1^{-1}$ is difficult.

\definecolor{firebrick2084028}{RGB}{208,40,28}
\definecolor{steelblue66156185}{RGB}{66,156,185}
\newenvironment{samplelegend}[1][]{%
        \begingroup

        \csname pgfplots@init@cleared@structures\endcsname
        \pgfplotsset{#1}%
    }{%
        % draws the legend:
        \csname pgfplots@createlegend\endcsname
        \endgroup
    }%
        \def\addlegendimage{\csname pgfplots@addlegendimage\endcsname}
\pgfplotsset{
cycle list={%
{draw=black,mark=star,solid},
{draw=black, mark=square,solid},%densely dashed}, 
{draw=black,mark=+,solid},%dashdotted}, %every mark/.append style={rotate=90},
{black,mark=o},}}
\begin{figure}[t]
  \centering
  \newcommand{\mywidth}{0.35\textwidth} 

    \setlength\tabcolsep{15.0pt} 
  \begin{tabular}[t]{cc}
  \includegraphics[width=\mywidth]{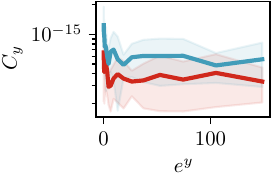} &
  \includegraphics[width=\mywidth]{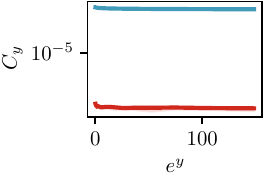} \\
  (a) $g_\text{aff}$ & (b) $g$\\
      \multicolumn{2}{c}{   
  \begin{tikzpicture}
  \begin{samplelegend}[legend columns=4,legend style={align=left,draw=none,column sep=2ex},legend entries={$P^{\text{Newton}}$, ${\psi}^{\text{opt}}_{x, \bar{y}}$}]
    \addlegendimage{line width=0.5mm, firebrick2084028}
    \addlegendimage{line width=0.5mm, steelblue66156185}
  \end{samplelegend}
  \end{tikzpicture}
  }
  \end{tabular}
	\caption{Compare $P^{\text{Newton}}$ from Prop.~\ref{prop:ideal_precond} and ${\psi}^{\text{opt}}_{x, \bar{y}}$ from Prop.~\ref{prop:ideal_sep_rep} on ridge regression but with different outer problems. We show the efficiency constant $C_y$ in $\log$ space under different $y$. (a) When the outer problem is affine, both strategies can achieve a small efficiency constant $C_y$ around machine accuracy. (b) When the outer problem is quadratic, the Newton preconditioner achieves the super efficiency while the ${\psi}^{\text{opt}}_{x, \bar{y}}$ has a large constant $C_y$.}
  \label{fig:opt_ridge_comp}
  \end{figure}

\definecolor{mediumblue4243192}{RGB}{42,43,192}
\definecolor{mediumseagreen8018079}{RGB}{80,180,79}
\definecolor{sandybrown25515670}{RGB}{255,156,70}
\definecolor{violet25595255}{RGB}{255,95,255}

        \def\addlegendimage{\csname pgfplots@addlegendimage\endcsname}
\pgfplotsset{
cycle list={%
{draw=black,mark=star,solid},
{draw=black, mark=square,solid},%densely dashed}, 
{draw=black,mark=+,solid},%dashdotted}, %every mark/.append style={rotate=90},
{black,mark=o},}}
\begin{figure}[t]
  \centering
  \newcommand{\mywidth}{0.35\textwidth} 
    \setlength\tabcolsep{15.0pt} 
  \begin{tabular}[t]{cc}
	
	\includegraphics[width=\mywidth]{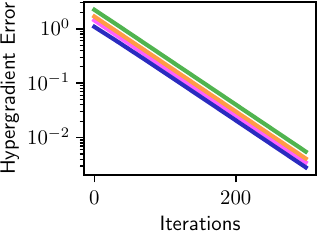} &
  \includegraphics[width=\mywidth]{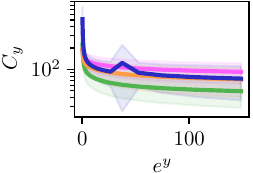} \\

  (a) bad preconditioner & (b) general case\\
  \multicolumn{2}{c}{   
  \begin{tikzpicture}
  \begin{samplelegend}[legend columns=5,legend style={align=left,draw=none,column sep=0.5ex},legend entries={Vanilla, $P^{\text{diag}}$, $\psi^{\text{exp}}_{x}$, $\psi^{\text{diag}}_{x}$}]
    \addlegendimage{line width=0.5mm, violet25595255}
    \addlegendimage{line width=0.5mm, mediumseagreen8018079}
    \addlegendimage{line width=0.5mm, mediumblue4243192}
    \addlegendimage{line width=0.5mm, sandybrown25515670}
  \end{samplelegend}
  \end{tikzpicture}
  }\\
  \end{tabular}
	\caption{Compare $P^{\text{diag}}$, $\psi^{\text{exp}}_{x}$, $\psi^{\text{diag}}_{x}$ on ridge regression with the outer problem $g$. (a) We show the hypergradient errors of different strategies in log space over the number of iterations when having a bad preconditioner. It turns out that $\psi^{\text{exp}}_{x}$ could be a better choice in this setting. (b) We show the efficiency constant $C_y$ of each strategy in log space under different $y$. Although reparameterization could perform better in some cases, $P^{\text{diag}}$ in general is the best choice.}
  \label{fig:general_ridge_comp}
  \end{figure}

\section{Numerical Experiments}
\label{sec:numerics}

We illustrate our findings on regression and classification supervised learning problems.
Bilevel programming is used to compute hyper-parameters $y$ controlling the regularization function. 
The inner and outer problems are of the form:
\begin{equation}
\begin{aligned}
    g(x)& = L(A_\text{val} x, b_\text{val}),\\
    f(x, y)& = L(A_\text{tr}x, b_\text{tr}) + \mathcal{R}(x, y),
\end{aligned}
\label{eq:num_prob}
\end{equation}
with $L(z, b) = \sum_i \ell(z_i, b_i)$, and the loss $\ell$ depends on the task (regression or classification).
Here $A_\text{tr} \in \R^{M \times d_x}$, $A_{\text{val}} \in \R^{N \times d_x}$ are the train and test design matrix respectively, 
$b_\text{tr} \in \R^M$ and $b_{\text{val}} \in \R^N$ are the train and test labels (restricted to $\{-1, 1\}$ for classification).
The coefficients $x \in \R^{d_x}$ are the weight parameters of the predictor.
The value of $y$ determines the structure of $F_1(x, y)$, for instance 
large $y$ lead to $F_1(x, y)$ being diagonally dominant.
The regularization functional is a ridge 
penalty $\mathcal{R}(x, y) \coloneqq \frac{1}{2} \sum_i^{d_x} e^{y_i}x_i^2$, so that $y$ introduces a feature-dependent penalization~\citep{pedregosa2016hyperparameter}.

Inspired by Prop.~\ref{prop:ideal_precond}, the first two strategies we consider are Newton preconditioner $P^{\text{Newton}}(x, y) \coloneqq F_1(x, y)$ and the diagonal preconditioner $P^{\text{diag}}(x, y) \coloneqq \text{diag}(F_1(x, y))$.
Similarly to Prop.~\ref{prop:supereffi-1d}, we also consider an exponential reparameterization defined as $\psi^{\text{exp}}_{x}(z) \coloneqq \text{sign}(x)\exp(z)$.
The last strategy is a separable diagonal reparameterization $\psi^{\text{diag}}_{x}(z, y) \coloneqq [\text{diag}(F_1(x, y))]^{-1} z$.
We run a gradient descent $x_{k} = x_{k-1} - \tau_k \nabla_x f(x_{k-1},y)$ with a proper step size $\tau_k > 0$ to attain an approximate root $x^k$ after $k$ steps.
All the experiments are run with Jax \citep{jax2018github} and can be found in \url{https://github.com/zhenzhang-ye/enhance_hypergradient}.

\subsection{Ridge Regression}

The first problem we consider is a regularized ridge regression, obtained using $\ell(z, b) = (z - b)^2$ in Eq.~\eqref{eq:num_prob}.
The design matrix $A_{\text{tr}}$ and labels $b_{\text{tr}}$ are from the dataset \textbf{mpg} in LIBSVM \citep{CC01a}, which consists of $M=392$ data and $d_x=7$ features.
The $A_{\text{val}} \in \R^{M \times d_x}$ and $b_{\text{val}} \in \R^{M}$ in the outer problem is randomly generated.
The optimal $\x[\star][y]$ is computed using the linear solver from Jax.
To directly assess the estimation error on $\partial x^\star(y)$ (see Prop.~\ref{prop:outer-inner}), we also consider a special case where the outer problem is affine, denoted by $g_{\text{aff}}$.

Since $F$ is bijective in this case, the  reparameterization $\psi^{\text{opt}}_{x, \bar{y}}(z, y) \coloneqq -[F_1(x, y)]^{-1}F(z, \bar{y}) + x$ from Prop.~\ref{prop:ideal_sep_rep} is readily applicable.
We first compare it to $P^{\text{Newton}}$ under different values of $y$.
We ran the experiments 10 times for each $y$ to remove the effect of randomness.
Fig.~\ref{fig:opt_ridge_comp}(a) shows that when the outer problem is $g_{\text{aff}}$, the two strategies both achieve very small efficiency constant $C_y$.
This agrees with our finding that both of the algorithms can achieve super efficiency if $D$ is $0$.
On the other hand, the advantage of Newton's preconditioner shows up when having the general $g$.
As shown in Fig.~\ref{fig:opt_ridge_comp}(b), the efficiency constant $C_y$ of $P^{\text{Newton}}$ is independent of the outer problem, while ${\psi}^{\text{opt}}_{x, \bar{y}}$ fails to attain a small constant $C_y$.

To cover cases where $F_1^{-1}$ is unavailable or expensive to apply, we compare $P^{\text{diag}}$, $\psi^{\text{exp}}_{x}$ and ${\psi}^{\text{diag}}_{x}$.
In Fig.~\ref{fig:general_ridge_comp} (a), we show a special case where  $\psi^{\text{exp}}_{x}$ works the best even with an arbitrary (non-affine) outer problem $g$.
The reason is that the $F_1$ is not diagonally dominated, and the diagonal preconditioner is not efficient anymore.
To this end, we compare the three strategies with different $y$ and show the results in Fig.~\ref{fig:general_ridge_comp}(b).
All three strategies improve the hypergradient estimator compared to the vanilla one when the outer problem is quadratic $g$.
The average performance of $P^{\text{diag}}$ is the best, but there can be cases where the two reparameterizations are the best choices.

\definecolor{mediumblue4243192}{RGB}{42,43,192}
\definecolor{mediumseagreen8018079}{RGB}{80,180,79}
\definecolor{sandybrown25515670}{RGB}{255,156,70}
\definecolor{violet25595255}{RGB}{255,95,255}
\def\addlegendimage{\csname pgfplots@addlegendimage\endcsname}
\pgfplotsset{
cycle list={%
{draw=black,mark=star,solid},
{draw=black, mark=square,solid},%densely dashed}, 
{draw=black,mark=+,solid},%dashdotted}, %every mark/.append style={rotate=90},
{black,mark=o},}}
\begin{figure}[t]
  \centering
  \newcommand{\mywidth}{0.35\textwidth} 
    \setlength\tabcolsep{15.0pt} 
  \begin{tabular}[t]{cc}
	\includegraphics[width=\mywidth]{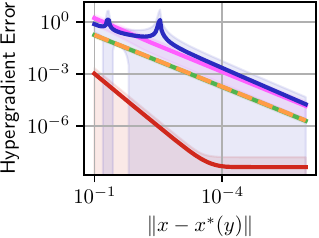} &
  \includegraphics[width=\mywidth]{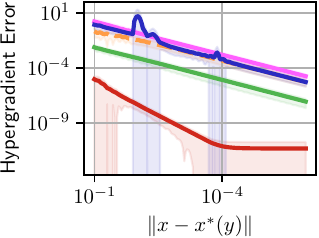} \\

  (a) Small $y$ & (b) Large $y$\\
  \multicolumn{2}{c}{   
  \begin{tikzpicture}
  \begin{samplelegend}[legend columns=6,legend style={align=left,draw=none,column sep=0.5ex},legend entries={\tiny Vanilla, \tiny $P^{\text{Newton}}$, \tiny $P^{\text{diag}}$, \tiny $\psi^{\text{exp}}_{x}$, \tiny $\psi^{\text{diag}}_{x}$}]
    \addlegendimage{line width=0.5mm, violet25595255}
    \addlegendimage{line width=0.5mm, firebrick2084028}
    \addlegendimage{line width=0.5mm, mediumseagreen8018079}
    \addlegendimage{line width=0.5mm, mediumblue4243192}
    \addlegendimage{line width=0.5mm, dashed, sandybrown25515670}
  \end{samplelegend}
  \end{tikzpicture}
  }\\
  \end{tabular}
	\caption{Comparison of different strategies on hypergradient error in log space over approximated root for logistic regression. $P^{\text{Newton}}$ always achieve the super efficiency. (a) With a small $y$, the performances of $P^{\text{diag}}$ and $\psi^{\text{diag}}_{x}$ are nearly the same. $\psi^{\text{exp}}_{x}$ performs worse than the vanilla one in some situations. (b) The performance of $P^{\text{diag}}$ improves thanks to the large $y$ which leads to a diagonally dominated $F_1$. The performances of two reparameterizations are the similar, both better than the vanilla one.}
  \label{fig:general_logistic_comp}
  \end{figure}

\subsection{Logistic Regression}
We now consider classification using logistic regression, using $\ell(z, b) = \text{log}(1 + \exp(-zb))$.
The design matrices and labels are from \textbf{liver-disorder} in LIBSVM with 145 training and 200 testing data.
There are 5 features and 2 classes of each data.
We ran all experiments 10 times with different random seeds.
The root $x^\star(y)$ is attained by the \texttt{minimize} function from Jax with a tolerance of $10^{-15}$.

To assess the impact of the scale of  $y$ on $F_1$ (and hence on the estimator efficiency), we compare the strategies with $y$ drawn from two different uniform distributions.
The first notable result in Fig.~\ref{fig:general_logistic_comp} is that $P^\text{Newton}$ achieves super efficiency in both cases.
This comes without a surprise because our finding shows that the performance of Newton's preconditioner is independent of $F_1$ and the problems.
Additionally when $y$ is generated over $[-1, 1)$ in Fig.~\ref{fig:general_logistic_comp}(a),
the $F_1$ is away from diagonally dominated.
It makes the performances of $P^{\text{diag}}$ and $\psi_{x}^\text{diag}$ nearly the same.
The exponential reparameterization $\psi_{x}^\text{exp}$ performs even worse than the vanilla one.
However, when $y$ generated over $[3, 6)$ becomes larger, all the strategies outperform the vanilla one.
The difference between $P^{\text{diag}}$ and $\psi_{x}^\text{diag}$ is expected because the diagonal preconditioner leads to a small $\delta$ in Prop.~\ref{prop:delta_eq} in this case.
For larger values of $y$, the performance of $\psi_{x}^\text{exp}$ is improved as well.
A plausible explanation is that larger $y$ makes $F_1$ close to a diagonal matrix, in which case the exponential reparameterization achieves super efficiency. 

\section*{Conclusion}

In this paper, we have developed an in-depth local analysis of implicit differentiation for bilevel programming. Of particular interest are detailed expressions of the efficiency constant $C_y$ of several variations around the vanilla IFT formula. This highlights the key challenge in designing an efficient hypergradient formula. 
On the theoretical side, the question of the existence (let alone the computation) of super-efficient reparameterization is still mostly open. It corresponds to highly non-linear partial differential equations. Better exploiting the connection between reparametrization and preconditioning could also be fruitful, for instance, in designing parametric formulas that could be adapted to specific machine-learning problems. 

\section*{Acknowledgement}
ZY and DC were supported by the ERC Advanced Grant SIMULACRON. The work of GP was supported by the European Research Council (ERC project NORIA) and the French government under management of Agence Nationale de la Recherche as part of the ``Investissements d’avenir'' program, reference ANR-19-P3IA-0001 (PRAIRIE 3IA Institute). The work was done during ZY's internship at ENS - PSL University. We thank Zaccharie Ramzi for the fruitful discussions.

\bibliographystyle{abbrvnat}
\bibliography{ref}

\newpage
\appendix
\counterwithin*{equation}{section}
\renewcommand\theequation{\thesection\arabic{equation}}
\section*{Appendix}
\section{Proof of Propositions}
\begin{aproposition}[Hypergradient approximation]\label{prop1}
   If $\tilde\Omega$ is $C^1$ and consistent, then for all $\hat{x}$ and $y$
    \begin{align*}
    &\norm{
        \tilde{\Omega}(\hat x,y) - 
        \nabla h(y) 
    }
    \leq 
    C_y \norm{ \x[\star][y]-\hat x }
    + 
    \mathcal{O}( \norm{ \x[\star][y]-\hat x }^2 ), \\
     &   \quad \text{where} \quad  
        C_y \coloneqq C_y(\tilde\Omega) \coloneqq \opnorm{\tilde{\Omega}_1(\x[\star][y], y)}.
    \end{align*}
\end{aproposition}
\begin{proof}
    The Taylor expansion of $\tilde{\Omega}(\hat{x}, y)$ at $(x^\star(y), y)$ yields:
    \begin{equation*}
        \tilde{\Omega}(\hat{x}, y) = \tilde{\Omega}(x^\star(y), y) + \tilde{\Omega}_1(x^\star(y), y)(x^\star(y) - \hat{x}) + \mathcal{O}(\norm{x^\star(y) - \hat{x}}^2).
    \end{equation*}
    Because $\tilde{\Omega}$ is consistent, we have
    \begin{equation*}
    \begin{aligned}
    \norm{
        \tilde{\Omega}(\hat x,y) - 
        \nabla h(y) 
    } &= \norm{\tilde{\Omega}_1(x^\star(y), y)(\hat{x} - x^\star(y)) + \mathcal{O}(\norm{x^\star(y) - \hat{x}}^2)}\\
    &\leq \norm{\tilde{\Omega}_1(x^\star(y), y)}_\text{op}\norm{(x^\star(y) - \hat{x})} +\mathcal{O}(\norm{x^\star(y) - \hat{x}}^2).
    \end{aligned}
    \end{equation*}
\end{proof}

\begin{aproposition}[Jacobian of estimation]\label{prop2}
Assuming $g$ and $F$ are smooth, one has 
\begin{equation*}
    \Jomg{\xy} = g_{21}{\xy} + \Psi_1 \xy g_1 \xy + \Psi \xy g_{11} \xy, 
\end{equation*}
\begin{equation*}
\begin{aligned}
\text{ with } \Psi_1 \xy = -[F_{12} \xy ]^\top[F_1 \xy[]{}]^{-1} + \bigl([F_1 \xy[]{}]^{-1} F_{11} \xy[]{} [F_1 \xy[]{}]^{-1}F_2 \xy[]{}\bigr)^\top.
\end{aligned}
%\label{eq:JPsi} % \label{eq:Jomg}
\end{equation*}
\end{aproposition}
\begin{proof}
    Recall that $\Omega(x, y) = g_2(x, y) + \Psi(x, y) g_1(x, y)$ with $\Psi(x, y) = -[F_2(x,y)]^\top[F_1(x, y)]^{-1}$. Using the chain rule to compute the Jacobian w.r.t. $x$, we have:
    \begin{equation*}
    \Omega_1 = g_{21}(x, y) + \Psi_1(x, y) g_1(x, y) + \Psi(x, y) g_{11}(x, y).\\
    \end{equation*}
    Using the fact that the Jacobian of $[F(x)]^{-1}$ is $-[F(x)]^{-1}F_1(x)[F(x)]^{-1}$, we can get $\Psi_1(x, y)$ as defined.
\end{proof}

\begin{aproposition}[IFT efficiency]\label{prop3}
    One has
    $$
        C_y(\Omega) \leq \| g_{21}(x^\star(y), y) + [\partial x^\star(y)]^\top g_{11} \|_{\infty}  + 
             \norm{g_1(x^\star(y), y)}_\infty C_y(\Psi)
    $$
    where $\norm{H}_\infty := \sup_{x,y} \norm{H(x,y)}_\text{\upshape op}$.
    Hence, if $g$ is of the form $g(x, y) = ax + m(y)$, then $\Omega$ is super-efficient if $\Psi$ is super-efficient.
    If, furthermore, $F $ is of the form $A x + M(y)$, then $\Omega$ is super-efficient.
\end{aproposition}
\begin{proof}
    Recall the definition of $C_y(\Omega)$ in Prop.~\ref{prop1}, we have:
    \begin{equation*}
    \begin{aligned}
        C_y(\Omega) &= \opnorm{\Omega_1(x^\star(y), y)}\\
        &=\opnorm{g_{21}{\xy[*][y]} + \Psi_1 \xy[*][y] g_1 \xy[*][y] + \partial x^\star(y) g_{11} \xy[*][y]} \\
        &\leq \opnorm{g_{21}{\xy[*][y]}+ [\partial x^\star(y)]^\top g_{11} \xy[*][y]} + \opnorm{\Psi_1 \xy[*][y] g_1 \xy[*][y]}\\
        &\leq \norm{g_{21}{\xy[*][y]}+ [\partial x^\star(y)]^\top g_{11} \xy[*][y]}_\infty + \opnorm{\Psi_1 \xy[*][y]} \norm{g_1 \xy[*][y]}_\infty\\
    \end{aligned}
    \end{equation*}
    If $g$ is affine, we have $g_{21}(x, y) = g_{11}(x, y) \equiv 0$ and the first term becomes $0$.
    If furthermore, $F(x, y) = Ax + M(y)$, $C_y(\Psi) \equiv 0$, which implies $C_y(\Omega) \equiv 0$.
\end{proof}

\begin{aproposition}[Preconditioned estimation]\label{prop4}
    $\Omega^P$ is consistent and 
    \begin{equation*}%\label{eq:JomgP}
        \Jomg[P]{\xy} = \Omega_1(\tilde{x}, y)
        E^P\xy,
        % (I_{d_x} - P(y)^{-1}F_1\xy[*][y]).
    \end{equation*}
    where $E^P \xy \coloneqq I_{d_x} - [P\xy]^{-1} F_1 \xy + [P\xy]^{-1} P_1 \xy[][] [P\xy]^{-1} F\xy $.
\end{aproposition}
\begin{proof}
    The root of $F$ reads $F(x^\star(y), y) =0$. Recall the preconditioning step $\tilde{x}:=x - [P(x,y)]^{-1}F(x, y)$. Plugging $x = x^\star(y)$, we have $\tilde{x} = x^\star(y)$. We thus have $\Omega^{P}(x^\star(y), y) = \Omega(x^\star(y), y)$.

    Because $\Omega^P = \Omega(x - [P(x, y)]^{-1}F(x, y), y)$, using the chain rule and the fact that the Jacobian of $[P(x)]^{-1}$ is $-[P(x)]^{-1}P_1(x)[P(x)]^{-1}$, we have:
    \begin{equation*}
    \begin{aligned}
    \Omega_1^P(x, y) = \Omega_1(x - [P(x, y)]^{-1}F(x, y), y)(I_{d_x} - [P(x, y)]^{-1} F_1(x, y) + [P(x,y)]^{-1}P_1(x, y) [P(x, y)]^{-1}F(x, y)).
    \end{aligned}
    \end{equation*}
\end{proof}

\begin{aproposition}[Newton-like preconditioner]\label{prop5}
    For $P(x,y) = F_1(x,y)$, 
    $\Omega^P$ is super-efficient.
%$\Jomg[P]{\xy[*][y]} =0$.
\end{aproposition}

\begin{proof}
Note that $F(x^\star(y), y) = 0$ in this case. Thus, we have $E^P\xy[*][y] = I_{d_x} - [P(x^\star(y), y)]^{-1}F_1(x^\star(y), y) = 0$, which implies $\Omega^P_1(x^\star(y), y)=0$. 
\end{proof}

\begin{aproposition}\label{prop6} $\Omega^{\phi}$ is consistent and 
\begin{align*}   
\omg[\phi]{\xy} &\coloneqq g_2\xy + \Psi^{\phi}\xy g_1\xy,  \\
    \Psi^{\phi}\xy & \coloneqq \phi_2(z, y)^\top - U^\phi\xy^\top [V^\phi\xy]^{-1}, \\ 
    U^\phi(x, y) &\coloneqq F_2 \xy +F_1\xy \phi_2(z,y)+ [\phi_1(z,y)]^{-1} \phi_{21} (z,y) F \xy, \\
    V^\phi(x, y) &\coloneqq [\phi_1(z,y)]^{-\top} \phi_{11}(z, y) F(x, y)  [\phi_1(z, y)]^{-1} + F_1 \xy,
    \label{eq:V_def}
\end{align*}
with $z = \phi^{-1}(x, y)$. Additionally, denoting $x^\star \coloneqq x^\star(y)$ and $z^\star \coloneqq \phi^{-1}(x^\star, y)$:
\begin{equation*}
\begin{aligned}
   \Jomg[\phi](x^\star,y) &= D(y) + \Psi^{\phi}_1(x^\star,y) g_1(x^\star,y),
  \\
    D(y) & \coloneqq g_{21} (x^\star, y) + [\partial x^\star(y)]^\top g_{11} (x^\star, y), \\
    \Psi^{\phi}_1(x^\star,y) &= \Psi_1(x^\star, y) + C^\phi(y),\\
    C^\phi(y) &\coloneqq W^\phi (y)+ S^\phi (y) + T^\phi(y).
\end{aligned}
\end{equation*}
\begin{equation*}
\begin{aligned}
W^\phi(y)  &\coloneqq - [F_1(x^\star, y)]^{-1} [\phi_1(z^\star, y)]^{-\top} [\phi_{12}(z^\star, y)]^\top F_1(x^\star, y),\\
S^\phi(y) &\coloneqq \bigl([\phi_1(z^\star, y)]^{-\top}\phi_{11}(z^\star, y)F_1(x^\star, y)[\phi_1(z^\star, y)]^{-1}[F_1(x^\star, y)]^{-1}\phi_2(z^\star, y)\bigr)^\top,\\
T^\phi(y) &\coloneqq 
\bigl([F_1(x^\star, y)]^{-1}[\phi_1(z^\star, y)]^{-\top}\phi_{11}(z^\star, y)F_1(x^\star, y)[\phi_1(z^\star, y)]^{-1}[F_1(x^\star, y)]^{-1}F_2(x^\star, y)\bigr)^{\top}.
\end{aligned}
\end{equation*}
%\label{prop:rep_est}
\end{aproposition}
\begin{proof}
    After the change of variable, we have the inner problem $f(\phi(z, y), y)$ and the outer problem $g(\phi(z(y), y), y)$. Using the chain rule, the hypergradient is then:
    \begin{equation}\label{eq:hyper_phi}
    \begin{aligned}
        \nabla h(y) &= \nabla g(\phi(z(y), y), y)\\
        &= g_2(\phi(z(y), y), y) + [\partial z(y)]^\top [\phi_1(z(y), y)]^\top g_1(\phi(z(y), y), y) + [\phi_2(z(y), y)
        ]^\top g_1(\phi(z(y), y)).
    \end{aligned}
    \end{equation}
    
    Now, we turn to compute $\partial z(y)$. Computing the gradient of the inner problem w.r.t. $z$ and denoting the fixed point as $(z^\star, y)$ with $z^\star \coloneqq z^\star(y)$, we have the new equation:
    \begin{equation*}
        \tilde{F}(z^\star, y) \coloneqq \phi_1(z^\star, y)^\top F(\phi(z^\star, y), y) = 0.
    \end{equation*}
    Viewing it as a fixed-point equation on $(z^\star, y)$, we can apply IFT to get $\partial z^\star(y)$:
    \begin{equation*}
    \begin{aligned}
        \partial z^\star(y) &= -[\tilde{F}_1(z^\star, y)]^{-1}\tilde{F}_2(z^\star, y),\\
        \text{where }\tilde{F}_1(z^\star, y)&= \phi_{11}(z^\star, y)F(\phi(z^\star, y), y) + [\phi_1(z^\star, y)]^\top F_1(\phi(z^\star, y), y) \phi_1(z^\star, y),\\
        \tilde{F}_2(z^\star, y) &= [\phi_{21}(z^\star, y)]^\top F(\phi(z^\star, y), y) + [\phi_1(z^\star, y)]^\top F_1(\phi(z^\star, y), y) \phi_2(z^\star, y) + [\phi_1(z^\star, y) ]^\top F_2(\phi(z^\star, y), y).\\
    \end{aligned}
    \end{equation*}
    Plugging it back to Eq.~\eqref{eq:hyper_phi}, we can get $\Omega^{\phi}(\phi(z^\star(y), y), y)$. Because $\phi(z^\star(y), y) = x^\star(y)$ and $\phi$ is bijective, we recall the notation that $\phi(x^\star(y), y) \coloneqq \phi(\phi^{-1}(x^\star(y), y), y)$. $\Omega^\phi(x, y)$ is achieved.

    Now we turn to compute the Jacobian of $\Omega^\phi(x, y)$ w.r.t. $x$. We introduce a new notation $\tD_1 F(x, y)$ meaning the derivative of $F$ w.r.t. the first variable.
    \begin{equation}
    \begin{aligned}
        \Omega_1^\phi(x, y) &= g_{21}(x, y) + \Psi_1^{\phi}(x, y)g_1(x, y) + \Psi^\phi(x, y) g_{11}(x, y),\\
    \text{where } \Psi_1^\phi(x, y) &= \tD_1\phi_{2}^\top(z, y) - [U^\phi_1(x, y)]^\top[V^\phi(x, y)]^{-1} + [U^\phi(x, y)]^\top [V^\phi(x, y)]^{-1} V_1^\phi(x, y) [V^\phi(x, y)]^{-1},\\
    U^\phi_1(x, y) &= F_{12}(x, y)  + F_{11}(x, y) \phi_2(z, y) + F_1(x, y) \tD_1\phi_{2}(z, y) + [\phi_1(z, y)]^{-1}\bigl([\phi_{12}(z, y)]^\top F_1(x, y)\bigr)^\top\\
    & + \tD_1\phi_1^{-1}(x, y)\phi_{21}(z, y) F(x, y) + \phi^{-1}(x, y)\tD_1\phi_{21}(z, y) F(x, y),\\
    V^\phi_1(x, y) &= F_{11}(x, y) + [\phi_1(z, y)]^{-\top}\phi_{11}(z, y) F_1(x, y)[\phi_1(z, y)]^{-1}\\
    & + \tD_1\phi_1^{-\top}(x, y)\phi_{11}(z, y) F(x, y)[\phi_1(z, y)]^{-1} 
    + [\phi_1(z, y)]^{-\top} \tD_1 \phi_{11} (z, y) F(x, y)[\phi_1(z, y)]^{-1} \\
    &+ [\phi_1(z, y)]^{-\top}\phi_{11}(z, y)F(x, y)\tD_1\phi^{-1}_1(z, y).
    \end{aligned}
    \label{eq:omega_1_xy}
    \end{equation}
    Note that this is true for any $(x, y)$. However, the terms $\tD_1\phi_{2}$, $\tD_1 \phi_1^{-1}$, $\tD_1\phi_{21}$, $\tD_1\phi_{11}$ are  complicated.
    Because we have a inversion $\phi(x, y) = \phi(\phi^{-1}(x, y), y)$ depending on $x$.
    Fortunately, when we evaluate it at the root $(x^\star, y)$ where $x^\star \coloneqq x^\star(y) = \phi(z^\star, y)$, $F(x^\star, y)= 0$. Thus, we have simplified $U_1^\phi$ and $V_1^\phi$:
    \begin{equation}
    \begin{aligned}
        U_1^\phi(x^\star, y) &= F_{12}(x^\star, y)  
        + F_{11}(x^\star, y) \phi_2(z^\star, y) 
        + F_1(x^\star, y) \tD_1\phi_{2}(z^\star, y) 
        + [\phi_1(z^\star, y)]^{-1}\bigl[[\phi_{12}(z^\star, y)]^\top F_1(x^\star, y)\bigr]^\top,\\
        V_1^\phi(x^\star, y) &= F_{11}(x^\star, y) 
        +[\phi_1(z^\star, y)]^{-\top}\phi_{11}(z^\star, y) F_1(x^\star, y)[\phi_1(z^\star, y)]^{-1}.
    \end{aligned}
    \end{equation}
    Plugging this back to $\Psi_1^\phi(x^\star, y)$, we have:
    \begin{equation}
    \begin{aligned}
        \Psi_1^\phi(x^\star, y) =& -\bigl(F_{12}(x^\star, y) + F_{11}(x^\star, y) \phi_2(z^\star, y) + [\phi_1(z^\star, y)]^{-1}[\phi_{12}(z^\star, y)]^\top F_1(x^\star, y)\bigr)^\top [F_1(x^\star, y)]^{-1}\\
         +&\bigl(F_2(x^\star, y) + F_1(x^\star, y) \phi_2(z^\star, y)\bigr)^\top 
        [F_1(x^\star, y)]^{-1}F_{11}(x^\star, y)[F_1(x^\star, y)]^{-1}\\
        +&\bigl(F_2(x^\star, y) + F_1(x^\star, y) \phi_2(z^\star, y)\bigr)^\top [\phi_1(z^\star, y)]^{-\top} \phi_{11}(z^\star, y) F_1(x^\star, y)[\phi_1(z^\star, y)]^{-1}[F_1(x^\star, y)]^{-1}\\
        =& \underbrace{-[F_{12}(x^\star, y)]^{\top}[F_1(x^\star, y)]^{-1}
        +\bigl([F_1(x^\star, y)]^{-1}F_{11}(x^\star, y)[F_1(x^\star, y)]^{-1}F_2(x^\star, y)\bigr)^\top}_{\Psi_1(x^\star, y)}\\
        +& \underbrace{-[F_1(x^\star, y)]^{-1} [\phi_1(z^\star, y)]^{-\top} [\phi_{12}(z^\star, y)]^\top F_1(x^\star, y)}_{W^\phi(y)}\\
        +& \underbrace{\bigl([\phi_1(z^\star, y)]^{-\top}\phi_{11}(z^\star, y)F_1(x^\star, y)[\phi_1(z^\star, y)]^{-1}[F_1(x^\star, y)]^{-1}\phi_2(z^\star, y)\bigr)^\top}_{S^\phi(y)}\\
        +& \underbrace{\bigl([F_1(x^\star, y)]^{-1}[\phi_1(z^\star, y)]^{-\top}\phi_{11}(z^\star, y)F_1(x^\star, y)[\phi_1(z^\star, y)]^{-1}[F_1(x^\star, y)]^{-1}F_2(x^\star, y)\bigr)^{\top}}_{T^\phi(y)}.
    \end{aligned}
    \label{eq:psi_1_xy}
    \end{equation}
    The $F_1(x^\star, y) \tD_1 \phi_2(z^\star, y)$ is canceled out in the first equation.
    The $F_{11}(x^\star, y) \phi_2(z^\star, y)$ is canceled out in the second equation.
    We want to emphasize that the above equation is only true at $(x^\star, y)$.
\end{proof}

\begin{aproposition}\label{prop7}
    Assume $x,y \in \mathbb{R}$ and 
    that $g(x, y)$ and $F(x, y)$ are linear w.r.t. $x$ but arbitrary on $y$. Then $\Omega^\phi$ is super-efficient if and only if for all $y$, 
    \begin{equation*}
        \frac{\phi_{12}(z^\star(y), y)}{\phi_1(z^\star(y), y)} - \frac{\phi_2(z^\star(y), y) \phi_{11}(z^\star(y), y)}{[\phi_1(z^\star(y), y)]^2}
        - \frac{F_2(x^\star(y), y) \phi_{11}(z^\star(y), y)}{F_1(x^\star(y), y) [\phi_1(z^\star(y), y)]^2}
        = \frac{g_{12}(x^\star(y), y)}{g_1(x^\star(y), y)} -\frac{F_{12}(x^\star(y), y)}{F_1(x^\star(y), y)}  
    %\label{eq:exp_1d}
    \end{equation*}
\end{aproposition}
\begin{proof}
Because $g$ and $F$ are linear on $x$, $g_{11}(x, y)$ and $F_{11}(x, y)$ are $0$.
The inversion in 1D case equals division and the transpose can be ignored.
Using these two facts and Prop.~\ref{prop6}, we have:
\begin{equation*}
\begin{aligned}
    D(y) &= g_{12}(x^\star(y), y)\\
    \Psi_1(x^\star(y), y) &= - \frac{F_{12}(x^\star(y), y)}{F_1(x^\star(y), y)}\\
    W^\phi(y) &=  - \frac{\phi_{12}(z^\star(y), y)}{\phi_1(z^\star(y), y)}\\
    S^\phi(y) &= \frac{\phi_{11}(z^\star(y), y)\phi_2(z^\star(y), y)}{[\phi_1(z^\star(y), y)]^2}\\
    T^\phi(y) &=\frac{\phi_{11}(z^\star(y), y)F_2(x^\star(y), y)}{F_1(x^\star(y), y) [\phi_{1}(z^\star(y), y)]^2}\\
\end{aligned}
\end{equation*}
Plugging everything into $\Psi_1^\phi(x^\star(y), y)$, we have:
\begin{equation*}
\begin{aligned}
    \Psi_1^\phi(x^\star(y), y) &= D(y) + \bigl(\Psi_1(x^\star(y), y) 
    + W^\phi(y) 
    + S^\phi(y) 
    + T^\phi(y)\bigr)g_1(x^\star(y), y)\\
    &= g_{12}(x^\star(y), y)
    - \frac{F_{12}(x^\star(y), y)}{F_1(x^\star(y), y)}g_1(x^\star(y), y)
    - \frac{\phi_{12}(z^\star(y), y)}{\phi_1(z^\star(y), y)} g_1(x^\star(y), y)\\
     &+\frac{\phi_{11}(z^\star(y), y)\phi_2(z^\star(y), y)}{[\phi_1(z^\star(y), y)]^2}g_1(x^\star(y), y)
     +\frac{\phi_{11}(z^\star(y), y)F_2(x^\star(y), y)}{F_1(x^\star(y), y) [\phi_{1}(z^\star(y), y]^2}g_1(x^\star(y), y)
\end{aligned}
\end{equation*}
Setting it to $0$ and dividing both side with $g_1(x^\star(y), y)$, we get what we want.

\end{proof}

\begin{aproposition}\label{prop8}
If $g$ is affine on $x$ (see Prop.\ref{prop3}), then super-efficient reparameterization $\phi(z, y)=\phi_0(z)$ exists and defines locally a 2-parameters family of maps.
If furthermore $F$ is linear of $x$, i.e. $F\xy = a(y) x + b$, where $a: \R \rightarrow \R$ and $b \in \R$, these super-efficient maps are of the form $\phi_0(z) = \alpha e^{\beta z} $ for $(\alpha,\beta) \in \R$.
\end{aproposition}
\begin{proof}
Because $g$ is affine on $x$, we have $g_{12}(x, y) = 0$.
$\phi(z, y) = \phi_0(z)$ reads that $\phi_{12}(z, y) = \phi_2(z, y)=0$.
We first show the existence of $\phi_0(z)$. Denoting that $phi_0(z^\star(y)) = x^\star(y)$ and from Prop.~\ref{prop7}, we have:
\begin{equation*}
\frac{F_2(x^\star(y), y) \phi_{0,11}(z^\star(y))}{F_1(x^\star(y), y) [\phi_{0, 1}(z^\star(y), y)]^2}
        = \frac{g_{12}(x^\star(y), y)}{g_1(x^\star(y), y)} -\frac{F_{12}(x^\star(y), y)}{F_1(x^\star(y), y)}  
\end{equation*}
This equation can be reformulated to:
\begin{equation*}
\begin{aligned}
    \phi_{0, 11} &= \mathcal{F}(\phi_0, \phi_{0, 1})\\
    \text{where } \mathcal{F}(\phi_0, \phi_{0, 1}) &= \left(\frac{g_{12}(\phi_0(z^\star(y)), y)}{g_1(\phi_0(z^\star(y)), y)} -\frac{F_{12}(\phi_0(z^\star(y)), y)}{F_1(\phi_0(z^\star(y)), y)} \right)\frac{F_1(\phi_0(z^\star(y)), y)[\phi_{0, 1}(z^\star(y), y)]^2}{F_2(\phi_0(z^\star(y)), y)}
\end{aligned}
\end{equation*}
Denoting $\phi_{0, 1} = \mathcal{G}(\phi)$ yields:
\begin{equation*}
\begin{aligned}
    &\phi_{0, 11} = \mathcal{G}(\phi) \mathcal{G}'(\phi) \\
    \Rightarrow& \mathcal{G}(\phi) \mathcal{G}'(\phi) = \mathcal{F}(\phi, \mathcal{G})\\
    \Rightarrow& \mathcal{G}'(\phi) = \mathcal{\tilde{F}}(\phi, \mathcal{G}(\phi)).
\end{aligned}
\end{equation*}
with $\mathcal{\tilde{F}}(\phi, \mathcal{G}(\phi)) = \mathcal{F}(\phi, \mathcal{G}) / \mathcal{G}(\phi)$. This becomes a first-order differential equation about $\mathcal{G}$. The Cauchy-Lipschtz theorem shows the existence of $\mathcal{G}$ with one parameter. Performing another integration we can get $\phi_0$ with one additional parameter.

If $F(x, y) = a(y)x + b$, we have $F_1(x, y) = a(y)$, $F_{2}(x, y) = a'(y)x$ and $F_{12}(x, y) =a'(y)$.
Plugging everything into Prop.~\ref{prop7}, we have:
\begin{equation*}
\begin{aligned}
    &\frac{a'(y) \phi(z, y) \phi_{11}(z, y)}{a(y) [\phi_1(z, y)]^2} = \frac{a'(y)}{a(y)}\\
    \Rightarrow & \phi(z, y) \phi_{11}(z, y) = [\phi_1(z, y)]^2
\end{aligned}
\end{equation*}
The solution of this second-order ODE should be $\phi(z, y) = \alpha e^{\beta z}$ for $(\alpha, \beta) \in \R^2$.
\end{proof}

\begin{aproposition}\label{prop9}
    The estimator $\Omega_{\text{\upshape loc}}^\psi$ is consistent and one has
    $$
        C_y( \Omega_{\text{\upshape loc}}^\psi )
        =
        C_y( \Omega^{\psi_{x^\star(y),y}} ).
    $$
\end{aproposition}
\begin{proof}
    As Prop.~\ref{prop6} is true for any bijective $\phi$, so $\Omega^{\psi}_{\text{loc}}$ is consistent.
    
    We then show $C_y( \Omega_{\text{\upshape loc}}^\psi ) = C_y( \Omega^{\psi_{x^\star(y),y}} )$.
    We use the same notation $\tD_1F(x, y)$ as in the proof of Prop.~\ref{prop6}.
    The derivations in Eq.~\eqref{eq:omega_1_xy} still holds when computing $\tD_1 \Omega^\psi_{\text{loc}}(x, y)$ and $\tD_1 \Omega^{\psi_{x^\star(y), y}}(x, y)$.
    Note that there are differences on $\tD_1 \phi_2$, $\tD_1 \phi_1^{-1}$, $\tD_1 \phi_{21}$, $\tD_1 \phi_{11}$ and $\tD_1 \phi_2$.
    In the case $\Omega^\psi_{\text{loc}}$, the derivative considers the dependency on $x$ from the parameter of $\psi$, the change of variable $(x, y)$ and the inversion of $\psi^{-1}$,
    while $\Omega^{\psi_{x^\star(y), y}}$ only considers the latter two.
    However, they agree when evaluating at $(x^\star(y), y)$ thanks to $F(x^\star(y), y) =0$.
    
\end{proof}

\begin{aproposition}\label{prop10}
    Let $\phi(z, y) =  \psi_{x, \bar{y}}(z, y) = R(x, y)Q(z, \bar{y}) + x$,  denoting $x^\star \coloneqq x^\star(y)$ and $\psi_{x, \bar{y}}(z^\star, y) = x^\star$.
    one has:
% $C_y^\phi$, $C_{xx}^\phi$ and $C_x^{\phi}$ here for the later discussion:
\begin{equation*}
\begin{aligned}
 W^{\phi}(y) =& - [F_1(x^\star, y)]^{-1} [R(x, y)]^{-\top}[Q_1(z^\star, \bar{y})]^{-\top} Q_1(z^\star, \bar{y}) [R_2(x^\star, y)]^{\top}F_1(x^\star, y),\\
S^{\phi}(y) =& \bigl([R(x, y)]^{-\top}[Q_1(z^\star, \bar{y})]^{-\top}Q_{11}(z^\star, \bar{y}) R(x, y) F_1(x^\star, y)\\
&[Q_1(z^\star, \bar{y})]^{-1} [R(x, y)]^{-1}[F_1(x^\star, y)]^{-1}Q(z^\star, \bar{y})R_2(x, y)\bigr)^\top, \\
 T^{\phi}(y) =&\bigl([F_1(x^\star, y)]^{-1}[R(x, y)]^{-\top}[Q_1(z^\star, \bar{y})]^{-\top}[Q_{11}(z^\star, \bar{y})]\\
 &R(x, y)F_1(x^\star, y)[Q_1(z^\star, \bar{y})]^{-1} [R(x, y)]^{-1}[F_1(x^\star, y)]^{-1} F_2(x^\star, y)\bigr)^\top,
\end{aligned}
\end{equation*}
for any $(x, \bar{y})$.
\end{aproposition}
\begin{proof}
    Because $\psi_{x, \bar{y}}$ is a special case of $\phi$, we can apply the results from Prop.~\ref{prop6} and \ref{prop9}. Computing $W^\phi(y)$, $S^\phi(y)$ and $T^\phi(y)$ require $\phi_1$, $\phi_{12}$, $\phi_2$ and $\phi_{11}$. Viewing $x, \bar{y}$ as constant, we have:
    \begin{equation*}
    \begin{aligned}
        \phi_1(z, y) &= \psi_{1, x, \bar{y}}(z, y) = R(x, y) Q_1(z, \bar{y})\\
        \phi_{12}(z, y) &= \psi_{12, x, \bar{y}}(z, y)=
        Q_1(z, \bar{y}) R_2(x, y)\\
        \phi_2(z, y) & = \psi_{2, x, \bar{y}}(z, y) =  Q(z, \bar{y}) R_2(x, y)\\
        \phi_{11}(z, y) &= \psi_{11, x, \bar{y}}(z, y) = Q_{11}(z, \bar{y})R(x, y).
    \end{aligned}
    \end{equation*}
    Plugging them into Prop.~\ref{prop6} to compute $W^\phi(y)$, $S^\phi(y)$ and $T^\phi(y)$, we get the desired results.
\end{proof}

\begin{aproposition}[Newton-like reparameterization]\label{prop11}
    We assume $g$ is of the form $g(x, y) = ax + m(y)$(see Prop.~\ref{prop3} for a discussion).
    Let $\psi_{x, \bar{y}}(z, y)=R(x, y)Q(z, \bar{y}) + x$ and let $F(x, y)$ be bijective on $x$ for all $y$. 
    For $R(x, y) = [F_1(x, y)]^{-1}$, $Q(z, \bar{y}) = -F(z, \bar{y})$, $\Omega^{\psi}_{\text{\upshape loc}}$ is super efficiency.
\end{aproposition}
\begin{proof}
    We need to examine $C_y(\Omega_{\text{\upshape loc}}^\psi)$, which is the same as $\Omega_{1}^{\psi_{x^\star(y), y}}(x^\star(y), y)$ from Prop.~\ref{prop9}.
    Therefore, we compute $Q_1$, $Q_{11}$, $R_2$ with $Q(z, y) = -F(z, y)$ and $R(x^\star(y), y) = [F_1(x^\star(y), y)]^{-1}$, we have:
    \begin{equation*}
    \begin{aligned}
        Q_1(z, y) &=  -F_1(z,y) \\
        Q_{11}(z, y) &= -F_{11}(z, y)\\
        R_2(x, y) &  = -\bigl([F_1(x^\star(y), y)]^{-1} [F_{12}(x^\star(y), y)]^\top [F_1(x^\star(y), y)]^{-1}\bigr)^\top\\
    \end{aligned}
    \end{equation*}
    Because $F$ is bijective in the assumption, given $x^\star(y)$ we can compute $z^\star(y)$:
    \begin{equation*}
        \begin{aligned}
            z^\star(y) = \psi_{x^\star(y), y}^{-1}(x^\star(y), y) =Q^{-1}(-[R(x^\star(y),y)]^{-1}(x^\star(y) - x^\star(y)), y)
            = F^{-1}(0, y) = x^\star(y).
        \end{aligned}
    \end{equation*}
    We thus have $Q(z^\star(y), y) = 0$.
    Using that$F_1(x^\star(y), y)$ is symmetric, and denoting $x^\star = x^\star(y)$ and $z^\star = z^\star(y)$, we have from Prop~\ref{prop10}:
    \begin{equation*}
        \begin{aligned}
            W^\phi(y) =& -[F_1(x^\star, y)]^{-1}[F_1(x^\star, y)]^\top [-F_1(z^\star, y)]^{-1} \\
            &[-F_1(z^\star, y)][-F_1(x^\star, y)]^{-1} [F_{12}(x^\star, y)]^\top [F_1(x^\star, y)]^{-1}F_1(x^\star, y)\\
            =& [F_{12}(x^\star, y)]^\top [F_1(x, y)]^{-1}\\
            S^\phi(y) =& 0\\
            T^\phi(y) =&\bigl( [F_1(x^\star, y)]^{-1}[F_1(x^\star, y)] [-F_1(z^\star, y)]^{-1}[-F_{11}(z^\star, y)]\\
            &[F_1(x^\star, y)]^{-1}F_1(x^\star, y)[-F_1(x^\star, y)]^{-1}F_1(x^\star, y) [F_1(x^\star, y)]^{-1} F_2(x^\star, y)\bigr)^\top\\
            =&-\bigl([F_1(z^\star, y)]^{-1} F_{11}(z^\star, y) [F_1(z^\star, y)]^{-1} F_2(x^\star, y) \bigr)^\top
        \end{aligned}
    \end{equation*}
    Because $z^\star = x^\star$, we have $\Psi_1^\phi(x^\star, y) = \Psi_1(x^\star, y) + W^\phi(y) + T^\phi(y) = 0$ with $\phi = \psi_{x^\star(y), y}$. $g(x, y) = ax + m(y)$ reads that $D(y) = 0$. In total, we have $\Omega_1^{\psi_{x^\star(y), y}}(x^\star(y), y) = 0$.
\end{proof}

% this might be a bit inconsistent. we should move it to g be affine and discuss omega_{loc}^\psi and a sign error on E^{R_2} = R_2 + F_1 F_21 F_1.
\begin{aproposition}
    Let $x^\star\coloneqq x^\star(y)$, $\psi_{x^\star, y} \coloneqq R(x^\star, y) Q(z, y) + x^\star$, and $z^\star$ satisfy $\psi_{x^\star, y}(z^\star, y) = x^\star$.
    Denoting $E^Q(y) \coloneqq Q(z^\star, y)  + F(x^\star, y)$,
    $E^{Q_1}(y) \coloneqq Q_1(z^\star, y) + F_1(x^\star, y)$,
    $E^{Q_{11}}(y) \coloneqq Q_{11}(z^\star, y) + F_{11}(x^\star, y)$,
    $E^R(y) \coloneqq R(x^\star, y) - [F_1(x^\star, y)]^{-1}$, 
    $E^{R_2}(y) \coloneqq R_2(x^\star, y) + [F_1(x^\star, y)]^{-1}F_{21}(x^\star, y)[F_1(x^\star, y)]^{-1}$.
    We have that 
    $
        C_y(\Psi^{\psi}_{\text{\upshape loc}}) = \mathcal{O}(\norm{E^Q}_{\text{\upshape op}},\norm{E^{Q_1}}_{\text{\upshape op}}, \norm{E^{Q_{11}}}_{\text{\upshape op}}, \norm{E^R}_{\text{\upshape op}}, \norm{E^{R_2}}_{\text{\upshape op}}).
    $
\end{aproposition}
\begin{proof}
    We need to compute $C_y(\Psi_{\text{loc}}^\psi)$, which is the same as $C_y(\Psi^{\psi_{x^\star, y}})$.
    Because $\Psi_1^{\psi_{x^\star(y), y}}(x^\star, y) =\Psi_1(x^\star, y) + W^{\psi_{x^\star(y), y}}(y) + S^{\psi_{x^\star(y), y}}(y) + T^{\psi_{x^\star(y), y}}(y)$ from Prop.~\ref{prop6} and $W^{\psi_{x^\star(y), y}}(y) = [F_{12}(x^\star, y)]^\top[F_1(x, y)]^{-1}$ when $E^{R}(y) = E^{Q_1}(y) = E^{R_2}(y) = 0$ from the proof of Prop.~\ref{prop11}.
    We first analyze the error on $W^{\psi_{x^\star(y), y}}$:
    \begin{equation*}
        \begin{aligned}
             W^{\psi_{x^\star(y), y)}}(y) =& - [F_1(x^\star, y)]^{-1} [R(x, y)]^{-\top}[Q_1(z^\star, y)]^{-\top} Q_1(z^\star, y) [R_2(x^\star, y)]^{\top}F_1(x^\star, y)\\
             = &[F_1(x^\star, y)]^{-1}(E^R(y) + [F_1(x^\star, y)]^{-1})^{-\top}(E^{Q_1}(y) - F_1(x^\star, y))^{-\top}(E^{Q_1}(y) - F_1(x^\star, y))\\
             &(E^{R_2}(y) - [F_1(x^\star, y)]^{-1}F_{21}(x^\star, y)[F_1(x^\star, y)]^{-1})^\top F_1(x^\star, y)\\
             = & [F_{12}(x^\star, y)]^\top [F_1(x^\star, y)]^{-1}\\
             &+(E^{Q_1}(y) - F_1(x^\star, y))^{-\top}(E^{Q_1}(y) - F_1(x^\star, y))[E^{R_2}(y)]^\top F_1(x^\star, y)\\
             &+\bigl([E^{Q_1}(y)]^{-\top}[E^{Q_1}(y)] + [-F_1(x^\star, y)]^{-\top}[E^{Q_1}(y)]\\
             &+ [E^{Q_1}(y)]^{-\top}[-F_1(x^\star, y)]\bigr)\bigl(-F_{21}(x^\star, y)[F_1(x^\star, y)]^{-1}\bigr)^\top\\
             &+[F_1(x^\star, y)]^{-1}[E^R(y)]^{-\top} (E^{Q_1}(y) - F_1(x^\star, y))^{-\top}(E^{Q_1}(y) - F_1(x^\star, y))(-F_{21}(x^\star, y)[F_1(x^\star, y)]^{-1})^\top\\
             &+ [F_1(x^\star, y)]^{-1} [E^R(y)]^{-\top}(E^{Q_1}(y) - F_1(x^\star, y))^{-\top}(E^{Q_1}(y) - F_1(x^\star, y))[-E^{R_2}(y)]^\top F_1(x^\star, y).
        \end{aligned}
    \end{equation*}
    Though it is complicated, the last four lines in the last equation are polynomial on $E^{Q_1}$, $E^{R_2}$ and $E^R$.
    Therefore, we have $\opnorm{W^{\psi_{x^\star(y), y}}(y)} = \opnorm{[F_{12}(x^\star, y)]^\top [F_1(x^\star, y)]^{-1}} + \mathcal{O}(\opnorm{E^{Q_1}}, \opnorm{E^{R_2}}, \opnorm{E^R})$.
    The same logic can be applied to $S^{\psi_{x^\star(y), y}}$ and $T^{\psi_{x^\star(y), y}}$. Putting them all together, we finish the proof.
\end{proof}

\begin{aproposition}[Comparison of the two methods]\label{prop13}
    Let $\phi$ be smooth and bijective, one has% and $\Jomg[P] \xy$ be defined as in Eq.~\eqref{eq:JomgP}
    \begin{align}
        [C_y(\Omega^\phi)]^2 - [C_y(\Omega^P)]^2 &\geq \langle U_+(y) v_P(y), U_-(y) v_P(y) \rangle \\
        [C_y(\Omega^P)]^2 - [C_y(\Omega^\phi)]^2 &\geq \langle V_+(y) v_\phi(y), V_-(y) v_\phi(y) \rangle
    \end{align}
    where, for $E^P$ as in Prop.~\ref{prop4}, $\Psi_1^\phi$ and $D$ as in Prop.~\ref{prop6}, 
    \begin{align*}
        U_\pm(y) &\coloneqq D(y) \pm D(y) E^P(x^\star(y), y) + \Psi_1^{\phi}(x^\star(y), y)g_1(x^\star(y), y) \pm \Psi_1(x^\star(y), y) g_1(x^\star(y), y) E^P(x^\star(y), y), \\
        V_\pm(y) &\coloneqq D(y) E^P(x^\star(y), y) \pm D(y) + \Psi_1(x^\star(y), y) g_1(x^\star(y), y) E^P(x^\star(y), y) \pm \Psi_1^{\phi}(x^\star(y), y)g_1(x^\star(y), y), \\ 
        v_{\omega}(y) &\coloneqq \argmax_{\lnorm{u}=1} \lnorm{\Jomg[\omega]{\xy[\star][y]}u}
        \text{ for } \omega \in \{P,\phi\}, 
    \end{align*}
%    Similarly, choosing $v\coloneqq \argmax_{\lnorm{u}=1} \lnorm{\Jomg[\phi]{\xy[\star][y]}}^2$, we have:
\end{aproposition}
\begin{proof}
    We derive $[C_y(\Omega^\phi)]^2 - [C_y(\Omega^P)]^2$ and the similar idea can be applied to derive $[C_y(\Omega^P)]^2 - [C_y(\Omega^\phi)]^2$.
    \begin{equation*}
        \begin{aligned}
        [C_y(\Omega^\phi)]^2 - [C_y(\Omega^P)]^2 &= \opnorm{\Omega^\phi_1(x^\star(y), y)}^2 - \opnorm{\Omega^P(x^\star(y), y)}^2 \\
        =&\opnorm{D(y) + \Psi_1^\phi(x^\star(y), y) g_1(x^\star(y), y)}^2 - \opnorm{\Omega_1(x^\star(y), y) E^P(x^\star(y), y)}^2\\
        =&\opnorm{D(y) + \Psi_1^\phi(x^\star(y), y)g_1(x^\star(y), y)}^2 - \opnorm{(D(y) + \Psi_1(x^\star(y), y) g_1(x^\star(y), y)) E^P(x^\star(y), y)}^2\\
        \geq &\langle \bigl(D(y) + \Psi_1^\phi(x^\star(y), y)g_1(x^\star(y), y) - (D(y) + \Psi_1(x^\star(y), y) g_1(x^\star(y), y)) E^P(x^\star(y), y)\bigr)v(y),\\
        &\bigl(D(y) + \Psi_1^\phi(x^\star(y), y)g_1(x^\star(y), y) + (D(y) + \Psi_1(x^\star(y), y) g_1(x^\star(y), y)) E^P(x^\star(y), y)\bigr)v(y)\rangle
        \end{aligned}
    \end{equation*}
    by choosing $v(y) = \argmax_{\norm{u}=1}\norm{\Omega^{P}(x^\star(y), y) u}$.
\end{proof}

% here is a typo, the error should be P and F_1. not F_1^{-1}.
\begin{aproposition}
For $\delta \coloneqq \norm{P\xy[\star][y] - F_1\xy[\star][y]}_\infty$, we have
$$
    [C_y(\Omega^\phi)]^2 - [C_y(\Omega^P)]^2 \geq \norm{(D(y)+ \Psi^{\phi}_1(x^\star(y), y)g_1(x^\star(y), y)) v_P}^2 + o(\delta),
$$
with $v_P$ as defined in Prop.~\ref{prop13}
\end{aproposition}
\begin{proof}
    Recall that $E^P(x^\star(y), y) = I_{d_x} - [P(x^\star(y), y)]^{-1}F_1(x^\star(y), y) = [P(x^\star(y), y)]^{-1}(P(x^\star(y), y) - F_1(x^\star(y), y))$. Continuing deriving from Prop.~\ref{prop13}, we have:
    \begin{equation*}
    \begin{aligned}
    [C_y(\Omega^\phi)]^2 - [C_y(\Omega^P)]^2 \geq& \norm{(D(y) + \Psi_1^\phi(x^\star(y), y) g_1(x^\star(y), y))v_P}^2 - \norm{(D(y) + \Psi_1(x^\star(y), y) g_1(x^\star(y), y)) E^P(x^\star(y), y)v_P}^2\\
    \geq&\norm{(D(y) + \Psi_1^\phi(x^\star(y), y) g_1(x^\star(y), y))v_P}^2\\
    -& \opnorm{D(y) + \Psi_1(x^\star(y), y) g_1(x^\star(y), y)}^2 \opnorm{E^P(x^\star(y), y)}^2\norm{v_P}^2 \\
    \geq &\norm{(D(y) + \Psi_1^\phi(x^\star(y), y) g_1(x^\star(y), y))v_P}^2\\
    -& \opnorm{D(y) + \Psi_1(x^\star(y), y) g_1(x^\star(y), y)}^2 \opnorm{P(x^\star(y), y)}^2\norm{P(x^\star(y), y) - F_1(x^\star(y), y)}^2_\infty\norm{v_P}^2\\
     =& \norm{(D(y) + \Psi_1^\phi(x^\star(y), y) g_1(x^\star(y), y))v_P}^2 + o(\delta).
    \end{aligned}
    \end{equation*}
\end{proof}

% here we didn't define g^{(2)}.
\begin{aproposition}
For $\sigma \coloneqq \norm{g_1(x^\star(y), y)}_\infty C_y(\Psi^\psi_{\text{\upshape loc}})$
% $\sigma \coloneqq \norm{D+E^\psi g_1}$,
\begin{equation*}
    [C_y(\Omega^P)]^2 - [C_y(\Omega^\psi_{\text{\upshape loc}})]^2
     \geq\norm{(D(y) + \Psi_1(x^\star(y), y) g_1(x^\star(y), y)) E^P(x^\star(y), y)v_\phi}^2 - \norm{D(y)v_\phi}^2 +o(\sigma),
\end{equation*}
with $v_\phi$ as defined in Prop.~\ref{prop13} with $\phi = \Psi_{\text{\upshape loc}}^\psi$.
\end{aproposition}
\begin{proof}
    Continue from Prop.~\ref{prop13}, we have:
    \begin{equation*}
        \begin{aligned}
            [C_y(\Omega^P)]^2 - [C_y(\Omega^\psi_{\text{\upshape loc}})]^2 \geq& \norm{(D(y) + \Psi_1(x^\star(y), y) g_1(x^\star(y), y)) E^P(x^\star(y), y)v_\phi}^2\\
            -& \norm{(D(y) + \Psi_{1, \text{loc}}^\psi(x^\star(y), y) g_1(x^\star(y), y))v_\phi}^2\\
            \geq& \norm{(D(y) + \Psi_1(x^\star(y), y) g_1(x^\star(y), y)) E^P(x^\star(y), y)v_\phi}^2 - \norm{D(y)v_\phi}^2 \\
            -& \norm{\Psi_{1, \text{loc}}^\psi(x^\star(y), y) g_1(x^\star(y), y)) v_\phi}^2\\
            \geq &\norm{(D(y) + \Psi_1(x^\star(y), y) g_1(x^\star(y), y)) E^P(x^\star(y), y)v_\phi}^2 - \norm{D(y)v_\phi}^2 \\
            -& \norm{g_1(x^\star(y), y)}^2_\infty [C_y(\Psi_{\text{loc}}^\psi)]^2\norm{v_\phi}^2\\
            =& \norm{(D(y) + \Psi_1(x^\star(y), y) g_1(x^\star(y), y)) E^P(x^\star(y), y)v_\phi}^2 - \norm{D(y)v_\phi}^2 +o(\sigma).
        \end{aligned}
    \end{equation*}
\end{proof}
\vfill
\end{document}